\documentclass{article}

\usepackage{microtype}
\usepackage{graphicx}
\usepackage{subfigure}
\usepackage{booktabs}% for professional tables
\usepackage{caption}

\usepackage{hyperref}

\usepackage{amsmath}
\usepackage{amsthm}
\usepackage{amssymb}
\usepackage{mathtools}

\newcommand{\ra}[1]{\renewcommand{\arraystretch}{#1}}

\DeclarePairedDelimiter\ceil{\lceil}{\rceil}

\newcommand{\abs}[1]{\left\lvert#1\right\rvert}% absolute value
\newcommand{\argmin}[1]{\underset{#1}{ {\text{arg\,min}}}}%argmin
\newcommand{\betaapprox}{\beta^\text{approx}} % approximate explanation
\newcommand{\betahat}{\hat{\beta}}% empirical explanations
\newcommand{\binomial}[1]{\mathcal{B}(#1)}% binomial distribution
\newcommand{\card}[1]{\left|#1\right|}% card of a set without braces
\newcommand{\condexpec}[2]{\mathbb{E}\left[#1\middle|#2\right]}% conditional expectation
\newcommand{\condexpecunder}[3]{\mathbb{E}_{#3}\left[#1\middle|#2\right]}% conditional expectation
\newcommand{\dencst}{c}% constant in the denominator of \Sigma^{-1}
\newcommand{\defeq}{\vcentcolon =}% define equals
\newcommand{\Diff}{\,\mathrm{d}}% differential operator
\newcommand{\Distcos}{d_{\cos}}
\newcommand{\distcos}[2]{\Distcos(#1,#2)}% cosinus distance
\newcommand{\Exp}{\mathrm{exp}}% normal exponential
\newcommand{\Exps}{\mathrm{e}}% small exponential
\renewcommand{\exp}[1]{\Exp\left(#1\right)}% exponential of something
\newcommand{\exps}[1]{\Exps^{#1}}% e to the power something
\newcommand{\Expec}{\mathbb{E}}% symbol for expectation
\newcommand{\expec}[1]{\Expec\left[#1\right]}% expected value
\newcommand{\expecunder}[2]{\Expec_{#2}\left[#1\right]}% expectation under a certain law
\newcommand{\frobnorm}[1]{\norm{#1}_{\mathrm{F}}}% Frobenius norm
\newcommand{\Gammahat}{\hat{\Gamma}}% empirical Gamma
\newcommand{\Indic}{\mathbf{1}}% indicator
\newcommand{\indic}[1]{\Indic_{#1}}% indicator of something

\newcommand{\ig}{g} % integrated gradient
\newcommand{\igapp}{\ig^\text{approx}} % approximated integrated gradient
\newcommand{\norm}[1]{\left\lVert#1\right\rVert}% norm of something

\newcommand{\opnorm}[1]{\norm{#1}_{\mathrm{op}}}% operator norm of a matrix
\newcommand{\smallopnorm}[1]{\smallnorm{#1}_{\mathrm{op}}}% operator norm of a matrix
\newcommand{\Proba}{\mathbb{P}}% symbol for probability
\newcommand{\proba}[1]{\Proba\left (#1\right )}% probability of an event
\newcommand{\smallproba}[1]{\Proba (#1)}%
\newcommand{\probaunder}[2]{\Proba_{#2}\left(#1\right)}% probability under
\newcommand{\Reals}{\mathbb{R}}% real numbers
\newcommand{\Sigmahat}{\hat{\Sigma}}% empirical covariance matrix
\newcommand{\shape}{\mathcal{S}}% some shape
\newcommand{\smallexpec}[1]{\Expec[#1]}% expected value, small version
\newcommand{\smallnorm}[1]{\lVert#1\rVert}% norm of a vector
\newcommand{\Splus}{\mathcal{S}_+}% 
\newcommand{\Sminus}{\mathcal{S}_-}% 
\newcommand{\supp}[1]{\text{Supp}(#1)}% support of a random variable
\newcommand{\xibar}{\overline{\xi}}% replacement image

\theoremstyle{plain}
\newtheorem{theorem}{Theorem}%[section]
\newtheorem{proposition}{Proposition}%[section]
\newtheorem{lemma}{Lemma}%[section]
\newtheorem{corollary}{Corollary}%[section]
\theoremstyle{definition}
\newtheorem{definition}{Definition}%[section]
\newtheorem{remark}{Remark}%[section]
\makeatletter
\def\th@plain{%
  \thm@notefont{}% same as heading font
  \itshape % body font
}
\def\th@definition{%
  \thm@notefont{}% same as heading font
  \normalfont % body font
}
\makeatother

\usepackage[accepted]{icml2021}

\icmltitlerunning{What Does LIME Really See in Images?}

\begin{document}

\twocolumn[
\icmltitle{What Does LIME Really See in Images?}




\begin{icmlauthorlist}
	\icmlauthor{Damien Garreau}{uca}
	\icmlauthor{Dina Mardaoui}{polytech}
\end{icmlauthorlist}

\icmlaffiliation{uca}{Universit\'e C\^ote d'Azur, Inria, CNRS, LJAD, France}
\icmlaffiliation{polytech}{Polytech Nice}
%\icmlaffiliation{ed}{School of Computation, University of Edenborrow, Edenborrow, United Kingdom}
%
\icmlcorrespondingauthor{Damien Garreau}{damien.garreau@univ-cotedazur.fr}
%\icmlcorrespondingauthor{Eee Pppp}{ep@eden.co.uk}

%\icmlkeywords{Machine Learning, ICML}

\vskip 0.3in
]



\printAffiliationsAndNotice{}  % leave blank if no need to mention equal contribution

\begin{abstract}
The performance of modern algorithms on certain computer vision tasks such as object recognition is now close to that of humans.  
This success was achieved at the price of complicated architectures depending on millions of parameters and it has become quite challenging to understand how particular predictions are made. 
Interpretability methods propose to give us this understanding. 
In this paper, we study LIME, perhaps one of the most popular. 
On the theoretical side, we show that when the number of generated examples is large, LIME explanations are concentrated around a limit explanation for which we give an explicit expression.  
We further this study for elementary shape detectors and linear models. 
As a consequence of this analysis, we uncover a connection between LIME and integrated gradients, another explanation method.
More precisely, the LIME explanations are similar to the sum of integrated gradients over the superpixels used in the preprocessing step of LIME.
\end{abstract}

%%%%%%%%%%%%%%%%%%%%%%%%%%%%%%%%%%%%%%%%%%%%%%%%%%%%%%%%%%%%%%%%%%%

\section{Introduction}
\label{sec:introduction}

Deep neural networks and deep convolutional neural networks (CNN) in particular have changed the way computers look at images~\citep{schmidhuber2015deep}. 
Many specific tasks in computer vision such as character recognition and object recognition are now routinely achieved by personal computers with human-like accuracy. 
The success of these algorithms seems partly due to the great complexity of the models they encode, the most recent relying on hundreds of layers and millions of parameters. 

While the accuracy is often the only relevant metric for practitioners, there are numerous situations where one is not satisfied if the model is making good predictions for the wrong reasons.  
We would like to know \emph{why} the model makes a particular prediction.
Responding to this emerging need, many \emph{interpretability methods} have appeared in the last five years. 
Among them, \emph{model agnostic} methods aim to provide to the user meaningful insights on the inner working of a specific algorithm without making any specific assumption on the architecture of the model. 
We refer to \citet{Adadi_Berrada_2018,Guidotti_et_al_2018} and \citet{linardatos2021explainable} for recent review papers. 

\begin{figure}[ht!]
	\begin{center}
		\includegraphics[scale=0.4]{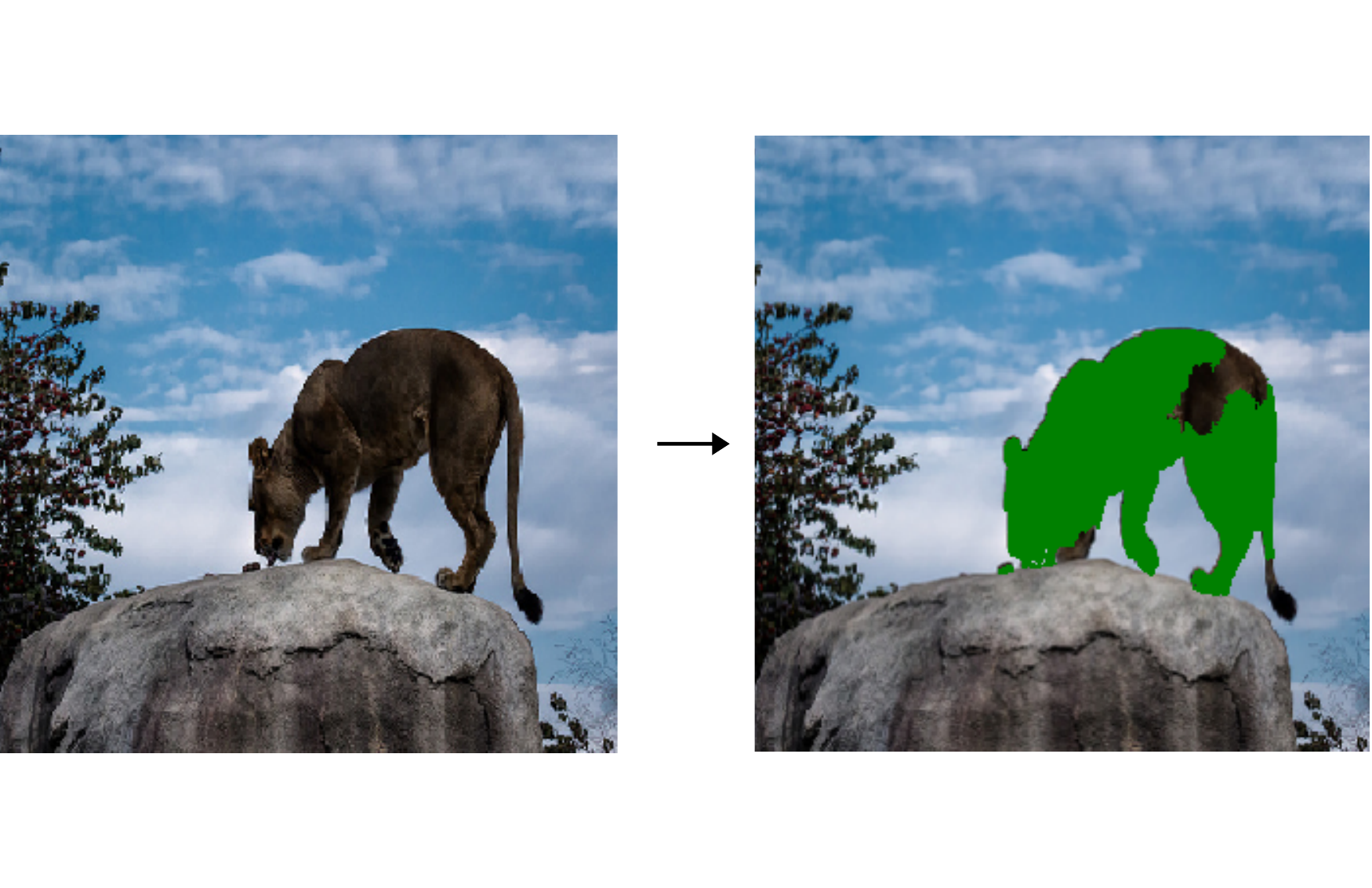}
	\end{center}
\vspace{-0.3in}
	\caption{\label{fig:example}Explaining a prediction with LIME. In this example, the function to be explained $f$ is the likelihood, according to the InceptionV3 network, that the input image $\xi$ contains a lion. %, the top label. 
	After a run of LIME with default parameters, the top five positive coefficients are highlighted in the right panel.}
\end{figure}

In this paper, we study the image version of LIME~\citep[Local Interpretable Model-agnostic Explanations,][]{ribeiro_et_al_2016}.
Let us recall briefly how it operates: in order to explain the prediction of a model $f$ for an example~$\xi$, LIME
\begin{enumerate}
	\item decomposes $\xi$ in $d$ superpixels, that is, small homogeneous image patches;
\item creates a number of new images $x_1,\ldots,x_n$ by \emph{randomly turning on and off} these superpixels;
\item queries the model, getting predictions $y_i=f(x_i)$;
\item builds a local weighted surrogate model $\betahat_n$ fitting the $y_i$s to the presence or absence of superpixels.
\end{enumerate}
Each coefficient of $\betahat_n$ is associated to a superpixel of the original image $\xi$ and, intuitively, the more positive the more important the superpixel is for the prediction at $\xi$ according to LIME. 
Generally, the user visualizes $\betahat_n$ by highlighting the superpixels associated to the top positive coefficients (usually five, see Figure~\ref{fig:example}).

The central question underlying this work is that of the soundness of LIME for explaining simple models: before using LIME on deep neural networks, are we sure that the explanations provided make sense for the most simple models? 
Can we guarantee it theoretically?

\paragraph{Contributions.}
Our contributions are the following:
\begin{itemize}
\item when the number of perturbed examples is large, the interpretable coefficients \textbf{concentrate with high probability around a vector $\beta$} that depends only on the model and the example to explain;
\item we provide an \textbf{explicit expression for $\beta$}, from which we gain some reassurance on LIME. In particular, the explanations are \textbf{linear} in the model;
\item for simple \textbf{shape detectors}, we can be more precise in the computation of $\beta$ and we show that \textbf{LIME provides meaningful explanations} in that case;
\item we can also compute $\beta$ for \textbf{linear models}. The limit explanation takes a very simple form: $\beta_j$ is \textbf{the sum of coefficients multiplied by pixel values on each superpixel};
\item as a consequence, we show experimentally that for models that are sufficiently smooth with respect to their inputs, the outputs of LIME are similar to the sum over superpixels of \textbf{integrated gradients}, another interpretability method.
\end{itemize}

\paragraph{Related work.}
While some weaknesses of LIME are well-known, in particular its vulnerability adversarial attacks \citep{slack2020fooling}, investigating whether the produced explanations do make sense is still an ongoing area of research. 
(see for instance \citet{narodytska2019assessing}). 
The present work follows the line of ideas initiated by \citet{garreau_luxburg_2020_aistats,garreau_luxburg_2020_arxiv} for the tabular data version of LIME and later extended to text data by \citet{mardaoui2020analysis}. 
In particular, our main result and its proof are similar to the theory laid out in these papers. 
The interesting differences come from the sampling procedure of LIME for images: there is no superpixel creation step in the text and tabular data version of the algorithm. 
Therefore, the exact expression of the limit explanations and the associated conclusions~differ.

\paragraph{Organization of the paper.}
We start by presenting LIME for images in Section~\ref{sec:lime}. 
Section~\ref{sec:main} contains our main results, which are further developed for simple models in Section~\ref{sec:expression}. 
Finally, we investigate the link between LIME and integrated gradients in Section~\ref{sec:approx}. 

%%%%%%%%%%%%%%%%%%%%%%%%%%%%%%%%%%%%%%%%%%%%%%%%%%%%%%%%%%%%%%%%%%

\section{LIME for Images}
\label{sec:lime}

From now on, we consider a model $f:[0,1]^D\to \Reals$ as well as a fixed example to explain $\xi\in[0,1]^D$. 
Hence $D$ denotes the number of pixels of the images on which $f$ operates. 
In practice, the inputs of $f$ are always $2$- or $3$-dimensional arrays. 
Of particular interest, grayscale images are usually encoded as $h\times w$ arrays, whereas RGB images are $h\times w\times 3$, with each channel corresponding to a primary color. 
We will see that it does not make a difference and our results can be read \emph{channel-wise} if there is more than one color channel. 

%%%%%%%%%%%%%%%%%%%%%%%%%%%%%%%%%%%%%%%%%%%%%%%%%%%%%%%%%%%%%%%%%

\subsection{Superpixels}
\label{sec:superpixels}

The first step of the LIME operation is to split $\xi$ into \emph{superpixels}. 
These are contiguous patches of the image that share color and / or brightness similarities. 
We refer to Figure~\ref{fig:sampling} for an illustration. 
In the text version of LIME, the counterpart of this superpixel decomposition is a local dictionary where each interpretable feature is a unique word of the text, 
whereas in the tabular version a complicated discretization procedure is needed. 

By default, LIME uses the \emph{quickshift} algorithm to produce these superpixels \citep{vedaldi2008quick}. 
In a nutshell, quickshift is a mode-seeking algorithm that considers the pixels as samples over a $5$-dimensional space ($3$ color dimensions and $2$ space dimensions).

\begin{figure*}[ht]
\begin{center}
\includegraphics[scale=0.45]{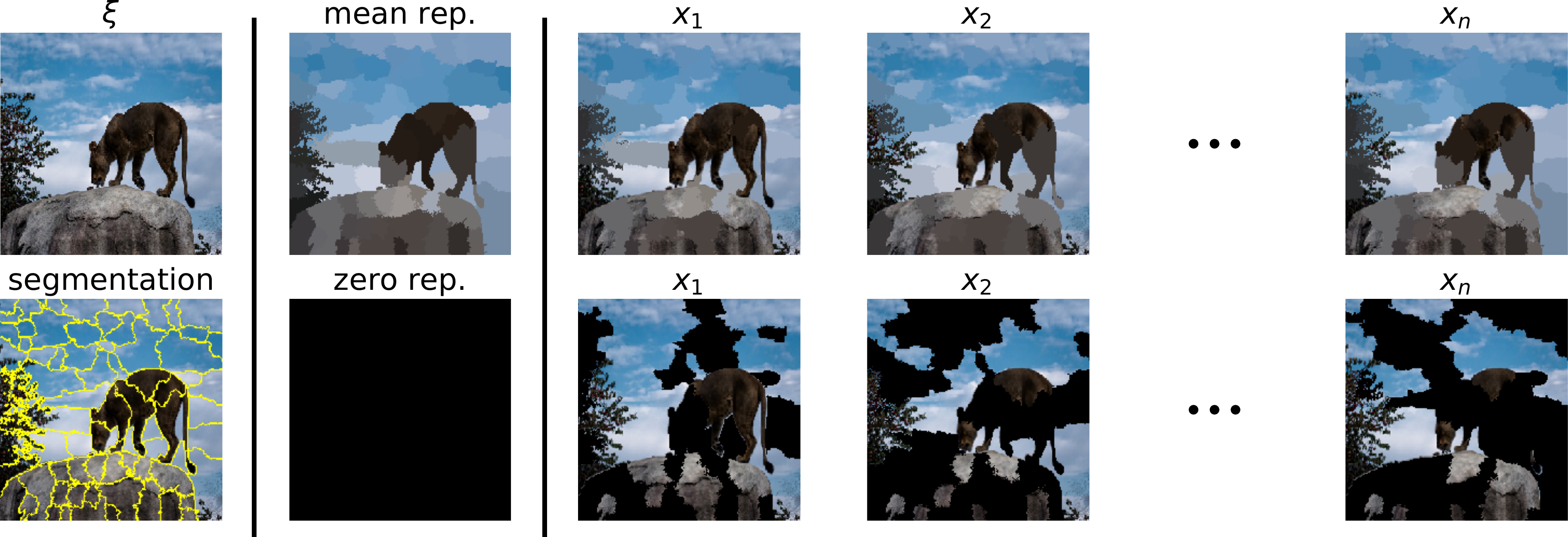}
\end{center}
\vspace{-0.1in}
\caption{\label{fig:sampling}Sampling procedure of LIME for images. The image to explain, $\xi$, is first split into $d$ superpixels (\emph{lower left corner}, here $d=72$). A replacement image $\xibar$ is computed, which is by default the mean of $\xi$ on each superpixel (\emph{top row}), see Eq.~\eqref{eq:def-xibar}. This replacement image can also be filled uniformly with a pre-determined color (\emph{bottom row:} replacement with the color black). Then, for each new generated example $x_i$ with $1\leq i\leq n$, the superpixels are randomly switched depending on the throw of $d$ independent Bernoulli random variables with parameter $1/2$. Thus LIME creates $n$ new images where key parts of $\xi$ disappear at random.}
\end{figure*}

% notation
For any $1\leq k\leq d$, we denote the $k$th superpixel associated to $\xi$ by $J_k$. 
Therefore, the $d$ subsets $J_1,\ldots,J_d$ form a partition of the pixels, that is, 
\[
J_1\cup \cdots \cup J_d = \{1,\ldots,D\} 
\enspace\text{ and }\enspace J_k\cap J_\ell = \emptyset \,\,\,\forall k\neq \ell 
\, .
\]
Note that, even though the superpixels are generally contiguous patches of the image, we do not make this assumption. 

%%%%%%%%%%%%%%%%%%%%%%%%%%%%%%%%%%%%%%%%%%%%%%%%%%%%%%%%%%%%%%%%%%%

\subsection{Sampling}

As we have seen in Section~\ref{sec:introduction}, one of LIME's key ideas is to create \emph{new examples} from $\xi$ by randomly replacing some superpixels of the image. 
By default, these chosen superpixels are replaced by the mean color of the superpixel, a procedure that we call \emph{mean replacement}. 
It is also possible to choose a specific color as a replacement image. 
We demonstrate the sampling procedure in Figure~\ref{fig:sampling} as well as the two possible choices for the replacement image.

Let us be more precise and let us assume that $\xi$ is fixed and $J_1,\ldots,J_d$ are given. 
The first step of the sampling scheme is to compute the replacement image $\xibar\in [0,1]^D$. 
If a given color $c$ is provided, then $\xibar_u=c$ for all $1\leq u\leq D$. 
If no color is provided, then the mean image is computed: for any superpixel $J_k$, we define $\xibar\in [0,1]^D$ by
\begin{equation}
\label{eq:def-xibar}
\forall u\in J_k, \quad \xibar_u = \frac{1}{\card{J_k}}\sum_{u\in J_k}\xi_u
\, .
\end{equation}
Of course, if the input images have several channels, the mean is computed on each channel. 

Then, for each $1\leq i\leq n$, LIME samples a random vector $z_i\in\{0,1\}^d$ where each coordinate of $z_i$ is i.i.d. Bernoulli with parameter $1/2$. 
Each $z_{i,j}$ corresponds to the activation ($z_{i,j}=1$) or inactivation ($z_{i,j}=0$) of superpixel $j$. 
We call the $z_i$s the \emph{interpretable features}. 
To be precise, for any given $i\in\{1,\ldots,n\}$, the new example $x_i\in [0,1]^D$ has pixel values given by
\begin{equation}
\label{eq:def-new-examples}
\forall u \in J_j, \quad x_{i,u} = z_{i,j}\xi_u + (1-z_{i,j})\xibar_u
\, .
\end{equation}
Again, if $\xi$ has several color channels, Eq.~\eqref{eq:def-new-examples} is written channel-wise. 
Note that $\xi$ corresponds to the vector $\Indic=(1,\ldots,1)^\top$ (all the superpixels of the image are activated).

%%%%%%%%%%%%%%%%%%%%%%%%%%%%%%%%%%%%%%%%%%%%%%%%%%%%%%%%%%%%%%%%%%%

\subsection{Weights}

Of course, the new examples $x_i$ can be quite different from the original image. 
For instance, if most of the $z_{i,j}$ are zero, then $x_i$ is close to $\xibar$. 
Some care is taken when building the surrogate model, and new examples are given a positive weight $\pi_i$ that takes this proximity into account. 
By default, these weights are defined by
\begin{equation}
\label{eq:def-weights}
\forall 1\leq i\leq n, \quad \pi_i \defeq \exp{\frac{-\distcos{\Indic}{z_i}^2}{2\nu^2}}
\, ,
\end{equation}
where $\nu > 0$ is a positive \emph{bandwidth parameter} equal to $0.25$ by default and $\Distcos$ is the \emph{cosine distance}. 
Namely, 
	\[
\forall u,v\in\Reals^d,\quad 	\distcos{u}{v}\defeq 1 - \frac{u^\top v}{\norm{u}\cdot \norm{v}} 
	\, .
	\]
We see that $\distcos{z_i}{\Indic}$ takes near zero values if most of the superpixels are activated, and values near $1$ in the opposite scenario, as expected.

An important remark is that the weights $\pi_i$ \textbf{depend only on the number of inactivated superpixels}. 
Indeed, conditionally to $z_i$ having exactly $s$ elements equal to zero, we have $z_i^\top\Indic = d-s$ and $\norm{z_i}=\sqrt{d-s}$. 
Since $\norm{\Indic}=\sqrt{d}$, using Eq.~\eqref{eq:def-weights}, we deduce that $\pi_i=\psi(s/d)$, where we defined 
\begin{equation}
\label{eq:def-psi}
\forall t \in [0,1],\quad \psi(t) \defeq \exp{\frac{-(1-\sqrt{1-t})^2}{2\nu^2}}
\, .
\end{equation}

%%%%%%%%%%%%%%%%%%%%%%%%%%%%%%%%%%%%%%%%%%%%%%%%%%%%%%%%%%%%%%%%%

\subsection{Surrogate Model}

The next stage of LIME is to build a \emph{surrogate model}. 
More precisely, LIME builds a linear model with the interpretable features $z_i$ as input and the model predictions $y_i\defeq f(x_i)$ as responses. 
This linear model, in the default implementation, is obtained by (weighted) ridge regression \citep{hoerl_1970}. 
Formally, the outputs of LIME for model~$f$ and image $\xi$ are given by 
\begin{equation}
\label{eq:main-problem}
\betahat_n^{\lambda} \in \argmin{\beta\in\Reals^{d+1} } \biggl\{ \sum_{i=1}^{n} \pi_i(y_i-\beta^\top z_i)^2 + \lambda \norm{\beta}^2 \biggr\}
\, ,
\end{equation}
where $\lambda >0$ is a regularization parameter. 
We call the coordinates of $\betahat_n^\lambda$ the \emph{interpretable coefficients}. 
By convention, the $0$th coordinate of $\betahat_n^\lambda$ is the intercept of the model. 
Some feature selection procedure can be used: we do not consider such extensions in our analysis and keep to the default implementation, which is ridge. 

Another important remark is the following: 
as in the text and tabular cases, LIME uses the default setting of \texttt{sklearn} for the regularization parameter, that is, $\lambda=1$. 
Hence the first term in Eq.~\eqref{eq:main-problem} is roughly of order $n$ and the second term of order $d$. 
Since we experiment in the large~$n$ regime ($n=1000$ is default) and with images split up in $\approx 100$ superpixels, we are in a situation where $n\gg d$. 
Therefore, \textbf{we can consider that $\lambda=0$ in our analysis and still recover meaningful results}. 
We will denote by $\betahat_n$ the solution of Eq.~\eqref{eq:main-problem} with $\lambda=0$, that is, ordinary least-squares. 

% final step
The final step of LIME for images is to display the superpixels associated to the top \emph{positive} coefficients of $\betahat_n^\lambda$ (usually five, see Figure~\ref{fig:example}). 
Part of what makes the method attractive to the practitioner is the ease with which one can read the results from one run of LIME just by looking at the highlighted part of the image.
Note that it is also possible to highlight the superpixels associated to the top \emph{negative} coefficients in another color, to see which parts of the image have a \emph{negative} influence on the prediction. 

%%%%%%%%%%%%%%%%%%%%%%%%%%%%%%%%%%%%%%%%%%%%%%%%%%%%%%%%%%%%%%%%%%%%%%

\section{Main Results}
\label{sec:main}

In this section we present our main results. 
Namely, the concentration of $\betahat_n$ around $\beta^f$ (Section~\ref{sec:concentration}) and the expression of $\beta^f$ as a function of~$f$ and other quantities (Section~\ref{sec:beta}). 

%%%%%%%%%%%%%%%%%%%%%%%%%%%%%%%%%%%%%%%%%%%%%%%%%%%%%%%%%%%%%%%%%%%%%

\subsection{Concentration of $\betahat_n$}
\label{sec:concentration}

When the number of new samples $n$ is large, we expect the empirical explanations provided by LIME to stabilize. 
Our first result formalizes this intuition. 

\begin{theorem}[Concentration of $\betahat_n$]
\label{th:concentration-betahat}
Assume that $f$ is bounded by a constant $M>0$ on $[0,1]^D$. 
Let $\epsilon > 0$ and $\eta\in (0,1)$. 
Let $d$ be the number of superpixels. 
Then, there exists $\beta^f\in\Reals^{d+1}$ such that, for every
\[
n \gtrsim \max(M,M^2) \epsilon^{-2} d^7\exps{\frac{4}{\nu^2}} \log\frac{8d}{\eta}
\, ,
\]
we have $\smallproba{\smallnorm{\betahat_n-\beta^f} \geq \epsilon}\leq \eta$. 
\end{theorem}

We refer to the appendix for a complete statement (we omitted numerical constants and the intercept for clarity). 
Intuitively, Theorem~\ref{th:concentration-betahat} means that when $n$ is large,~$\betahat_n$ stabilizes around $\beta^f$. 
Thus we can focus on $\beta^f$ to study LIME. 
The main limitation of Theorem~\ref{th:concentration-betahat} is the dependency on~$d$ and~$\nu$: the control that we achieve on $\smallnorm{\betahat_n-\beta^f}$ is quite poor whenever $d$ is too large or $\nu$ is too small. 
Note also that $\betahat_n$ is given by the \emph{non-regularized} version of LIME. 

Theorem~\ref{th:concentration-betahat} is quite similar to Theorem~1 in \citet{garreau_luxburg_2020_arxiv} and Theorem~1 in \citet{mardaoui2020analysis}, which are essentially the same result for the tabular data and the text data version of LIME. 
The rate of convergence is slightly better here, but this seems to be an artifact of the proof and we do not think that one should sample less when dealing with images. 

%%%%%%%%%%%%%%%%%%%%%%%%%%%%%%%%%%%%%%%%%%%%%%%%%%%%%%%%%%%%%%%%%%%%%

\subsection{Expression of $\beta^f$}
\label{sec:beta}

In this section we obtain the explicit expression of $\beta^f$. 
Before doing so, we need to introduce additional notation. 
From now on, we introduce the random variable $z\in\{0,1\}^d$ such that $z_1,\ldots,z_n$ are i.i.d. samples of $z$; it is the only source of randomness in the sampling and all expectations are taken with respect to it. 
We denote by $\pi$ and $x$ the associated weights and examples. 

\begin{definition}[$\alpha$ coefficients]
\label{def:alpha-coefficients}
Define $\alpha_0\defeq \expec{\pi}$ and, for any $1\leq p\leq d$, $\alpha_p \defeq \expec{\pi z_1\cdots z_p}$.
\end{definition}

Intuitively, when $\nu$ is large, $\alpha_p$ corresponds to the probability that exactly $p$ superpixels of $\xi$ are turned \emph{on}. 
Since the sampling scheme of LIME for images is completely symmetrical as well as the definition of the weights, we see that this probability does not depend on the exact set of indices, hence the definition of the $\alpha$ coefficients. 
We show in appendix that the expected covariance matrix of problem~\eqref{eq:main-problem} can be written with the first three $\alpha$ coefficients. 
Though Definition~\ref{def:alpha-coefficients} is identical to Definition~3 in \citet{mardaoui2020analysis}, the exact expression of the $\alpha$ coefficients is different in this case since the sampling procedure differs. 

\begin{proposition}[Computation of the $\alpha$ coefficients]
\label{prop:computation-alpha}
Let $d\geq 2$ and $0\leq p\leq d$. 
For any $\nu > 0$, it holds that 
\[
\alpha_p = \frac{1}{2^d} \sum_{s=0}^{d} \binom{d-p}{s}\psi\left(\frac{s}{d}\right)
\, ,
\]
where $\psi$ is defined as in Eq.~\eqref{eq:def-psi}. 
\end{proposition}

We prove Proposition~\ref{prop:computation-alpha} in the appendix. 
From the $\alpha$ coefficients, we then form the normalization constant 
\[
\dencst_d \defeq (d-1)\alpha_0\alpha_2 - d\alpha_1^2 +\alpha_0\alpha_1
\, ,
\]
and the $\sigma$ coefficients:

\begin{definition}[$\sigma$ coefficients]
For any $d\geq 2$ and $\nu > 0$, define
\[
\begin{cases}
	\sigma_1 &\defeq -\alpha_1
	\, , \\
	\sigma_2 &\defeq \frac{(d-2)\alpha_0 \alpha_2 - (d-1)\alpha_1^2 + \alpha_0\alpha_1}{\alpha_1-\alpha_2}\, , \\
	\sigma_3 &\defeq \frac{\alpha_1^2-\alpha_0\alpha_2}{\alpha_1-\alpha_2 }
	\, .
\end{cases}
\]
\end{definition}

We show in appendix that the \emph{inverse of the expected covariance matrix} associated to problem~\eqref{eq:main-problem} can be expressed with the help of the $\sigma$ coefficients and $\dencst_d$. 
With these notation in hand, we have:

\begin{proposition}[Expression of $\beta^f$]
\label{prop:computation-beta}
Under the assumptions of Theorem~\ref{th:concentration-betahat}, we have $\dencst_d > 0$ and, for any $1\leq j\leq d$, 
\begin{align*}
\beta^f_j &=
\dencst^{-1}_d\!\biggl[\sigma_1 \expec{\pi f(x)} 
+ \sigma_2 \expec{\pi z_j f(x)}  \\
&\qquad\qquad\qquad\qquad  +\sigma_3 \!\sum_{\substack{k=1 \\ k\neq j}}^d \expec{\pi z_k f(x)}\biggr]
\, .
\end{align*}
\end{proposition}

\begin{figure*}[ht!]
	\begin{center}
		\includegraphics[scale=0.27]{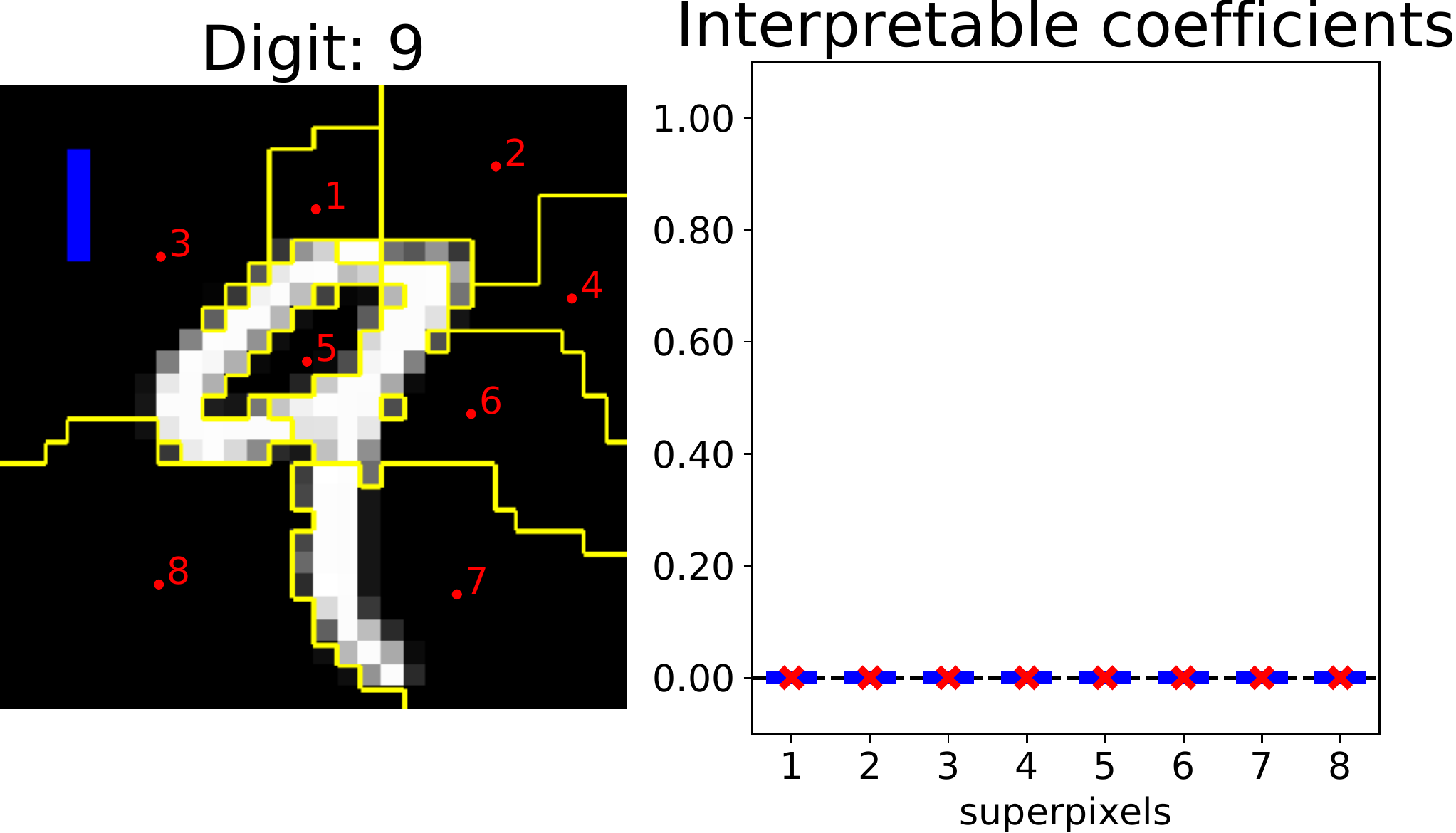}
		\includegraphics[scale=0.27]{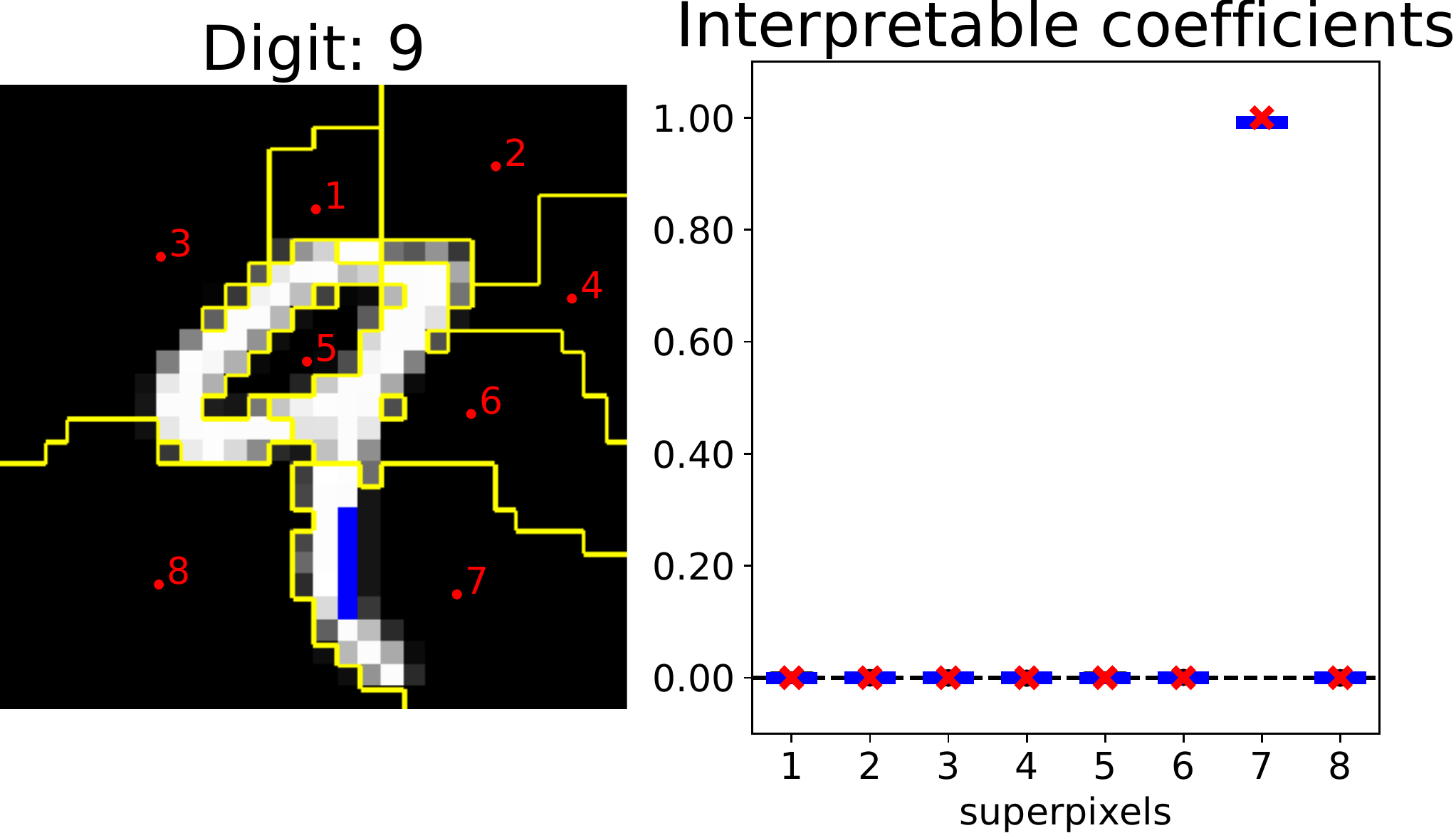}
		\includegraphics[scale=0.27]{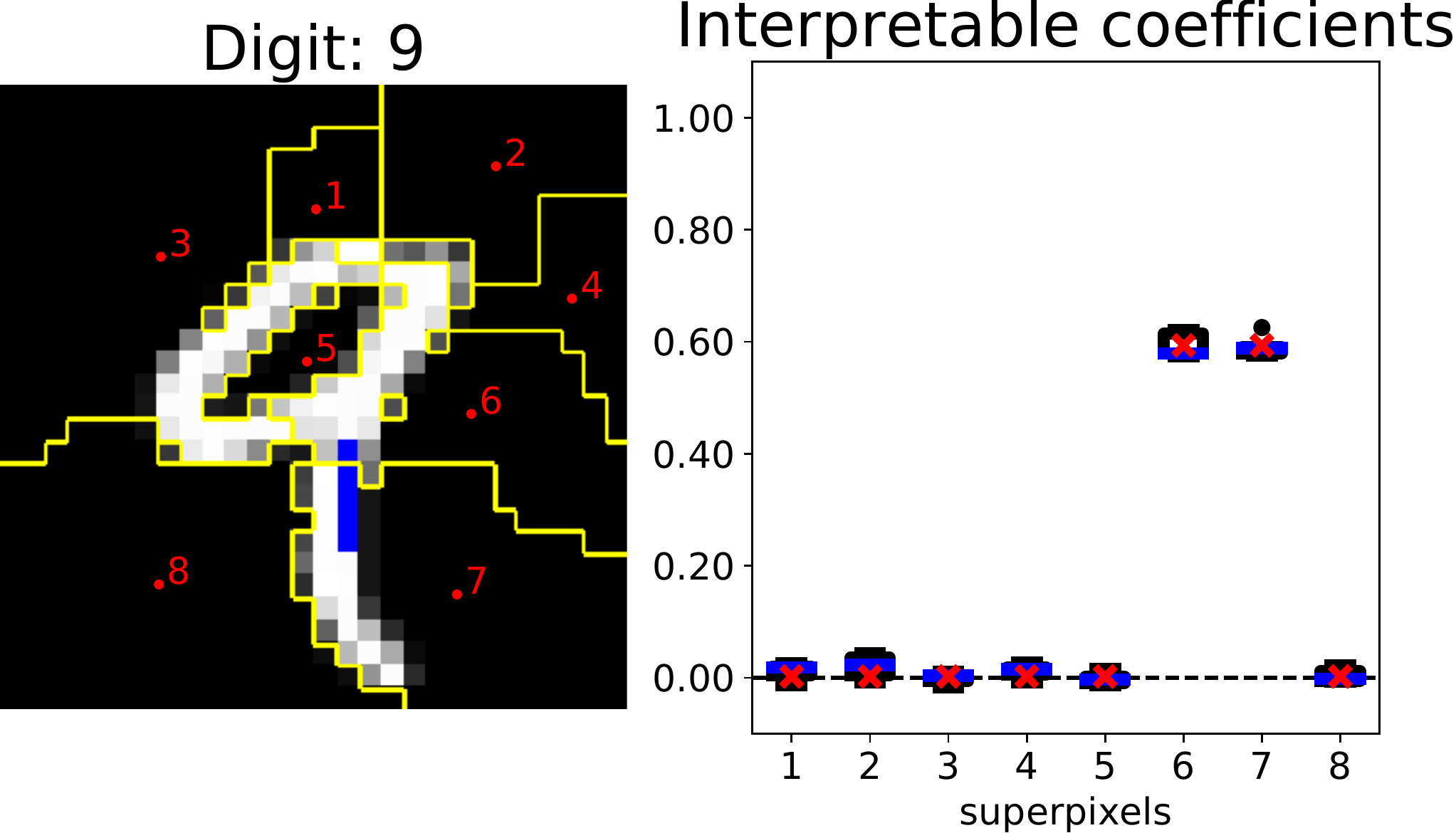}
	\end{center}
\vspace{-0.1in}
	\caption{\label{fig:shape-detector}In this figure, we show how the theoretical predictions of Proposition~\ref{prop:beta-computation-shape-detector} compare to practice. We considered a digit from the MNIST dataset~\citep{lecun1998gradient}. The function to explain takes value $1$ if all pixels marked in blue have value greater than $\tau=0.5$, $0$ otherwise. In each case, we ran LIME with $n=1000$ examples, default regularization $\lambda=1$ and zero replacement. We repeated the experiment five times, which gave the boxplot corresponding to the empirical values of the interpretable coefficients for each superpixel. 
		%We report the segmentation used by LIME in yellow. 
	The red crosses correspond to the predictions given by Proposition~\ref{prop:beta-computation-shape-detector}. We see that when the shape is split among $p$ superpixels, each one receives a coefficient approximately equal to $1/2^{p-1}$. }
\end{figure*}

We provide a detailed proof of Proposition~\ref{prop:computation-beta} as well as the expression of the intercept $\beta_0^f$ in the appendix. 
Let us note that Proposition~\ref{prop:computation-beta} is quite similar to Eq.~(6) in \citet{garreau_luxburg_2020_arxiv} and Eq.~(9) in \citet{mardaoui2020analysis}. 
We will see that many properties of $\beta^f$ are similar to the tabular and text case. 
Let us now present some immediate consequences of Proposition~\ref{prop:computation-beta}.

\textbf{Linearity of explanations. }
As in the tabular and the text setting, the mapping $f\mapsto \beta^f$ is \textbf{linear}. 
Thus for any model that can be decomposed as a sum, the explanations provided by LIME are \textbf{the sum of the explanations of individual models}. 
This is true up to noise coming from the sampling (quantified by Theorem~\ref{th:concentration-betahat}) and a small error due to the regularization, which is not taken into account in Theorem~\ref{th:concentration-betahat}. 

\textbf{Large bandwidth.}
Because of the weights $\pi$ and their complex dependency in the bandwidth $\nu$, it can be difficult to make sense of Proposition~\ref{prop:computation-beta} in the general case. 
It is somewhat easier when $\nu\to +\infty$. 
Indeed, we show in the appendix that $\dencst_d\to 1/4$, $\sigma_1\to -1/2$, $\sigma_2\to 1$, and $\sigma_3\to 0$. 
Moreover, $\pi\to 1$ almost surely. 
Therefore, by dominated convergence, the expression of $\beta^f$ simplifies to 
\begin{equation}
\label{eq:def-beta-infty}
\left(\beta_\infty^f\right)_j = 2\left( \condexpec{f(x)}{z_j=1}-\expec{f(x)}\right)
\, ,
\end{equation}
for any $1\leq j\leq d$. 
In other words, the explanation provided by LIME is proportional to the difference between the mean value of the model conditioned to superpixel $j$ being activated and the mean value of the model for $x$ sampled as explained previously.
It seems that Eq.~\eqref{eq:def-beta-infty} encompasses a desirable trait of LIME for images: when the bandwidth is large, the interpretable coefficient for superpixel $j$ takes large positive values if \textbf{in the vicinity of $\xi$, the model takes significantly larger values when this superpixel is present in the image}.
Of course, presence or absence of a given superpixel depends on the replacement scheme.
Eq.~\eqref{eq:def-beta-infty} hints that the explanation for superpixel $j$ could be near zero if $\xibar$ is close to $\xi$ on $J_j$, whereas $J_j$ is actually important for the prediction. 

%%%%%%%%%%%%%%%%%%%%%%%%%%%%%%%%%%%%%%%%%%%%%%%%%%%%%%%%%%%%%%%%%%%%%%%

\section{Expression of $\beta^f$ for Simple Models}
\label{sec:expression}

In this section, we get meaningful insights on the explanations provided by LIME for simple models. 

%%%%%%%%%%%%%%%%%%%%%%%%%%%%%%%%%%%%%%%%%%%%%%%%%%%%%%%%%%%%%%%%%%%%%%%

\subsection{Shape Detectors}

We start with a very simple model: a \emph{fixed shape detector} for gray scale images. 
To this extent, let $\shape\defeq \{u_1,\ldots,u_q\}$ be a set of $q$ distinct pixels indices, the \emph{shape}. 
Let $\tau \in (0,1)$ be a positive threshold. 
We define the associated shape detector 
\begin{equation}
\label{eq:def-shape-detector}
\forall x\in [0,1]^D,\quad f(x) \defeq \prod_{u\in\shape} \indic{x_u > \tau}
\, .
\end{equation}
Readily, $f(x)$ takes the value $1$ if the pixels of shape $\shape$ are lit up, and $0$ otherwise. 

% result
It is possible to compute $\beta^f$ in closed-form in this case. 
In fact, the result does not depend on the exact shape to be detected, but rather on how it intersect the LIME superpixels. 
Let us define 
\[
E \defeq  \{j\in\{1,\ldots,d\} \text{ s.t. } J_j\cap \shape \neq \emptyset\}
\, ,
\] 
the set of superpixels intersecting the shape $\shape$. 
We separate this set in two parts,
\[
E_+ \!\defeq\! \{j\in E \text{ s.t. } \xibar_j > \tau \}, \text{ and }
E_- \!\defeq\! \{j\in E \text{ s.t. } \xibar_j \leq \tau \}.
\]
Intuitively, $E_+$ (resp. $E_-$) contains superpixels that are brighter than $\tau$ on average (resp. darker). 
%We let $p\defeq \card{E_-}$ be the number of elements of~$E_-$. 
% TODO: what happens with zero replacement: assumption is also satisfied, maybe easier to understand
We also need to define 
\[
\Splus \!\defeq\! \{u\in \shape \text{ s.t. } \xi_u > \tau \},\text{ and }
\Sminus \!\defeq\! \{u \in \shape \text{ s.t. }\xi_u \leq \tau \}.
\]
Namely, $\Splus$ (resp. $\Sminus$) contains pixels in the shape that have a value greater (resp. smaller) than the threshold. 
 
% assumption
We now make the following assumption: 
\begin{equation}
\label{eq:assumption-shape}
\forall j\in E_+,\quad J_j\cap \Sminus = \emptyset
\, .
\end{equation}
Intuitively, Eq.~\eqref{eq:assumption-shape} means that there is no superpixel intersecting the shape such that the average value of the superpixel activates our detector without the detector being activated. 
This could happen for instance if the shape intersection with the superpixels is very dark but pixels around are much brighter. 
It is a reasonable assumption since superpixels are quite homogeneous in color and shape. 
Moreover, in the case of grayscale images with zero replacement, Eq.~\eqref{eq:assumption-shape} is always satisfied since $\tau >0$. 
Nevertheless, we provide a result that does not rely on Eq.~\eqref{eq:assumption-shape} in the appendix, which specializes into:

\begin{proposition}[Computation of $\beta^f$, shape detector]
\label{prop:beta-computation-shape-detector}
Let~$f$ be defined as in Eq.~\eqref{eq:def-shape-detector}. 
Assume that Eq.~\eqref{eq:assumption-shape} holds and let $p\defeq \card{E_-}$. 
Then, if there exists $j\in E_-$ such that $J_j\cap \Sminus \neq \emptyset$, $\beta^f=0$. 
Otherwise, for any $j\in E_-$,
\[
\beta_j^f \!=\! \dencst_d^{-1}\left\{\sigma_1\alpha_{p} + \sigma_2\alpha_{p} + (p-1)\sigma_3\alpha_{p} \!+\! (d-p)\sigma_3\alpha_{p+1} \right\}
\]
and for any $j\in \{1,\ldots,d\}\setminus E_-$,
\[
\beta_j^f \!=\! \dencst_d^{-1}\!\left\{\sigma_1\alpha_{p} \!+\! \sigma_2\alpha_{p+1} + p\sigma_3\alpha_{p} + (d-p-1)\sigma_3\alpha_{p+1} \right\}
\]
\end{proposition}

The accuracy of Proposition~\ref{prop:beta-computation-shape-detector} is demonstrated in Figure~\ref{fig:shape-detector}. 
Note that we use grayscale images for clarity (it is easier to define \emph{brightness} in this case), but Proposition~\ref{prop:beta-computation-shape-detector} could be adapted for RGB images. 

Note that Proposition~\ref{prop:beta-computation-shape-detector} is reminiscent of Proposition~3 in \citet{mardaoui2020analysis}. 
This is not a surprise, since both these results study LIME for models with analogous structures. 
However, the result differ since one has to consider the intersections of the superpixels in the present case. 
We now make two remarks, still focusing on grayscale images with zero replacement for simplicity: in that case, $E_+$ is always empty since $\tau$ is positive.

\textbf{Shape included in a single superpixel.}
In the simple case where the shape is detected and is totally included in superpixel~$k$, then it is straightforward from the definitions of the~$\alpha$ coefficients and $\dencst_d$ that $\beta^f_k=1$ and $\beta^f_j=0$ otherwise. 
This is a good property of LIME for images since superpixel~$k$ is the only relevant superpixel in this particular situation (see Figure~\ref{fig:shape-detector}). 

\textbf{Shape included in several superpixels.} 
The situation is slightly more complicated when $\shape$ intersects several superpixels (that is, $p>1$). 
Reading Proposition~\ref{prop:beta-computation-shape-detector}, we see that when $\nu$ is large, the coefficients for an intersecting superpixel is approximately $1/2^{p-1}$ and $0$ otherwise (see Lemma~1 in appendix for a more precise statement). 
That is, the importance given by LIME is evenly divided between superpixels that contain part of the shape. 
Again, this makes sense since they are the only relevant superpixels in this case (see Figure~\ref{fig:shape-detector}). 
For instance, if only two superpixels are involved, then the weight is roughly halved. 
We can note that this halving occurs even if a very small portion of the shape falls into one of the two superpixels. 

%%%%%%%%%%%%%%%%%%%%%%%%%%%%%%%%%%%%%%%%%%%%%%%%%%%%%%%%%%%%%%%%%%%%%

\subsection{Linear Model}

We now turn to \emph{linear models}. 
That is, $f$ is given by
\begin{equation}
\label{eq:linear-model}
\forall x\in [0,1]^D, \quad f(x) = \sum_{u=1}^{D} \lambda_u x_u
\, ,
\end{equation}
with $\lambda_1,\ldots,\lambda_D\in\Reals$ arbitrary coefficients. 
Note that we can adapt this definition if there is more than one color channel, by considering $3\times D$ coefficients. 

Let us compute $\beta^f$ when $f$ is linear. 
We show in the appendix the following:

\begin{proposition}[Computation of $\beta^f$, linear case]
\label{prop:beta-computation-linear}
	Assume that $f$ is defined as in Eq.~\eqref{eq:linear-model}. 
	Then, for any $1\leq j\leq d$,
	\[
	\beta_j^f = \sum_{u\in J_j} \lambda_u \cdot (\xi_u - \xibar_u)
	\, .
	\]
\end{proposition}

When $\xibar=0$, the coefficients take a very simple expression. 
Namely, the interpretable coefficient associated to superpixel $J_j$ is \textbf{the sum of the coefficients of $f$ multiplied by the pixel values on the superpixel}. 
If another replacement is chosen, then the \emph{normalized} values of the pixel is taken into account in this product. 
This seems to make a lot of sense: let us say that the coefficients of $f$ take large positive values on superpixel $j$.
Then LIME will give a high interpretable coefficient to this superpixel, unless the pixel values are small (or very close to the replacement color if another replacement is chosen). 
This is particularly visible in Figure~\ref{fig:linear}: the $\lambda$ coefficients take high values in the lower right corner (left panel). 
But since the $5$th superpixel contains only $0$-values pixels (middle panel), the interpretable coefficient given by LIME for this superpixel cancels out (right panel). 

It is also interesting to see that there is no bandwidth dependency in Proposition~\ref{prop:beta-computation-linear}. 
In a sense, this is to be expected since LIME is doing averages that are scale invariant if the function to explain is linear. 

\begin{figure}
\begin{center}
\includegraphics[scale=0.25]{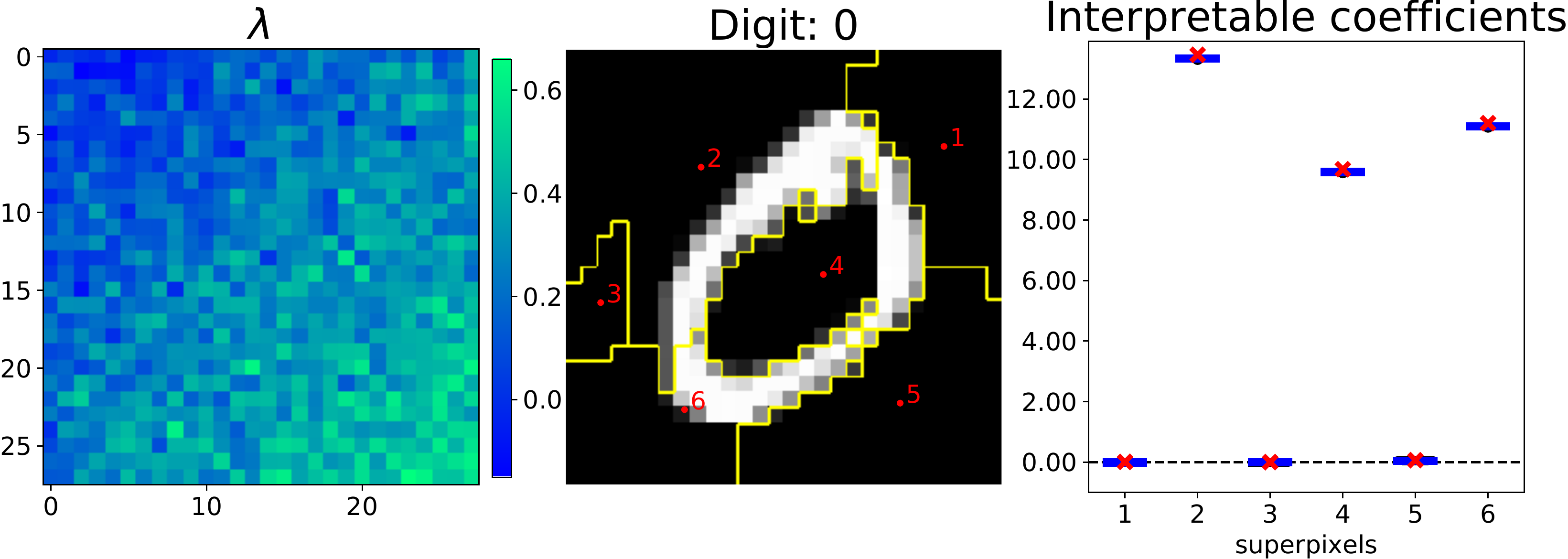}
\end{center}
\vspace{-0.1in}
\caption{\label{fig:linear}In this figure, we show how the theoretical predictions of Proposition~\ref{prop:beta-computation-linear} fare in practice. We consider a digit from the MNIST dataset. The function to explain is linear, with $\lambda_{i,j}$ proportional to $i+j$ with added white noise (heatmap in the \emph{left panel}). We ran LIME $5$ times with zero replacement and default parameters, the superpixels used are displayed in the \emph{middle panel}. We see that our predictions match perfectly. As predicted, $J_5$ receives a coefficient equal to zero whereas $f$ has high coefficients in this area, since the pixel values are equal to zero in this superpixel.}
\end{figure}

% link to existing results
Proposition~\ref{prop:beta-computation-linear} is similar in spirit to Eq.~(12) in \citet{mardaoui2020analysis}, where the interpretable coefficients provided by LIME for a linear model where also found to be (approximately) equal to the product of the coefficients of the linear model and the feature value. 

%%%%%%%%%%%%%%%%%%%%%%%%%%%%%%%%%%%%%%%%%%%%%%%%%%%%%%%%%%%%%%%%%%%

\section{Approximated Explanations}
\label{sec:approx}

If $f$ is regular enough, it can be written as in Eq.~\eqref{eq:linear-model} in the vicinity of $\xi$. 
If this is the case, an interesting question in light of the results of the previous section is the following: are the explanations given by Proposition~\ref{prop:beta-computation-linear} close to the LIME explanations? 
To put it plainly, can we approximately recover the LIME coefficients by summing the partial derivatives of $f$ at $\xi$ over the superpixels? 
We will see that the answer to this question is yes, with one caveat: simply taking the gradient does not always yield a satisfying linear approximation for complicated functions. 
We discuss linear approximations of an arbitrary function in Section~\ref{sec:linear-approx} before investigating empirically in Section~\ref{sec:experiments}.

%%%%%%%%%%%%%%%%%%%%%%%%%%%%%%%%%%%%%%%%%%%%%%%%%%%%%%

\subsection{Linear Approximation}
\label{sec:linear-approx}

The most natural linear approximation of a function is given by its Taylor expansion truncated at order one. 
Since we want to approximate $f(x)$, where $x$ is somewhere between~$\xi$ and $\xibar$, we could write, for instance, that 
$f(x) \approx f(\xi) + \nabla f(\xi)^\top (x-\xi)$.
There are two main objections in doing so in the present case. 
First, we do not expect $f$ to be linear between $\xi$ and $\xibar$, and taking just one gradient would lead to a poor approximation. 
We illustrate this behavior in Figure~\ref{fig:integrated-gradient} by computing the predictions across a straight line between~$\xi$ and $\xibar$.

\begin{figure}
\begin{center}
\includegraphics[scale=0.35]{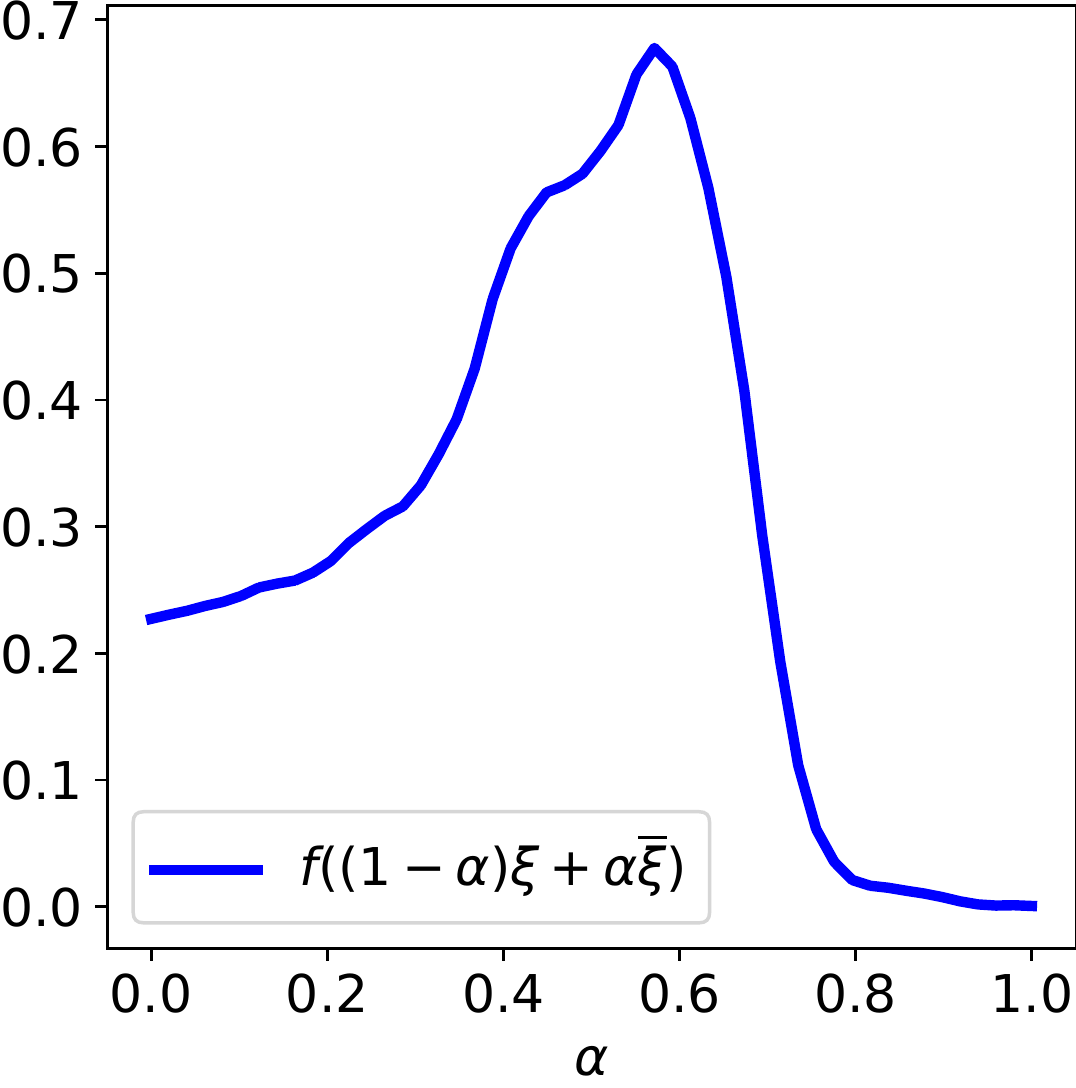}
\hfill
\includegraphics[scale=0.35]{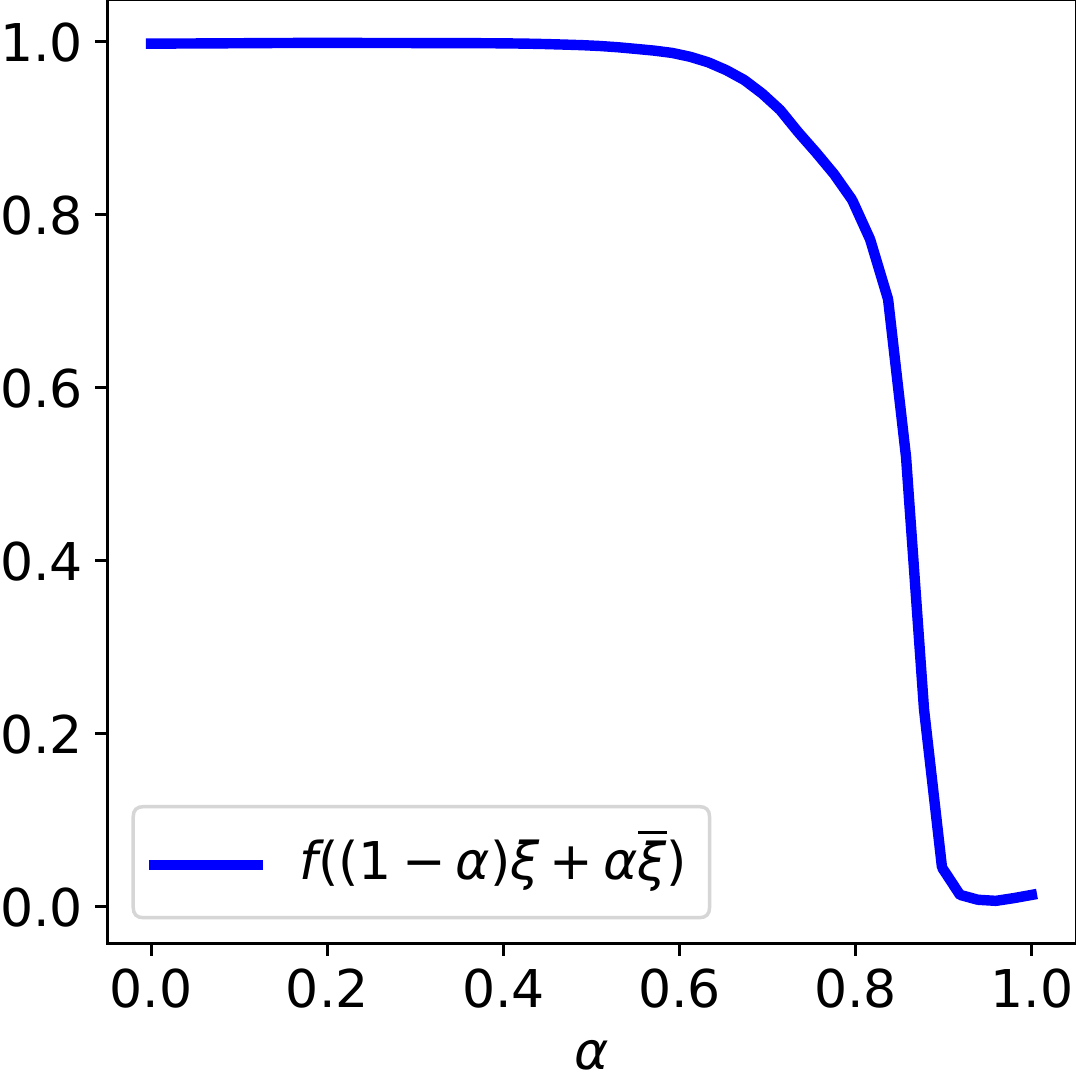}
\end{center}
\vspace{-0.1in}
\caption{\label{fig:integrated-gradient}Predictions given by the InceptionV3 network~\citep{szegedy2016rethinking} on a linear path between $\xi$ ($\alpha=0$) and $\xibar$ ($\alpha=1$). \emph{Left panel:} we see how the predicted values can vary along the path and why only considering the gradient at $\xi$ or $\xibar$ may not be a good idea to build a linear approximation. \emph{Right panel:} we see how the gradient can saturate when the network is very confident in the prediction. }
\end{figure}

Second, it is a well-known phenomenon in modern architectures that the gradient of the model with respect to the input can \emph{saturate} when the network is confident in the prediction for certain activation functions (see, for instance, \citet{krizhevsky2012imagenet}). 
Since from our point of view~$f$ is a black-box model, we do not have information on the activation functions (in fact, we do not even assume that~$f$ is a neural network).
Therefore gradients taken at $\xi$ or $\xibar$ can be zero, giving us essentially no information on the behavior of $f$ (see Figure~\ref{fig:integrated-gradient}). 

\begin{figure*}[ht!]
	\begin{center}
		\includegraphics[scale=0.45]{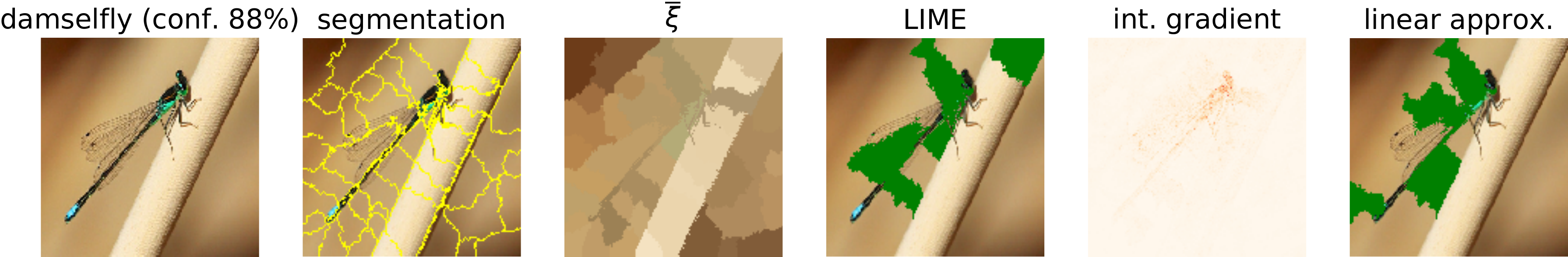}
		
		\hspace{0.3cm}\includegraphics[scale=0.45]{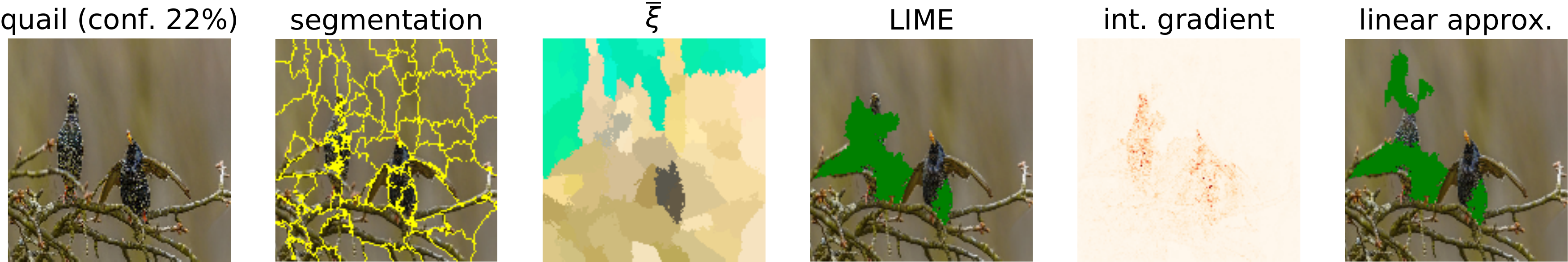}
	\end{center}
\vspace{-0.1in}
	\caption{\label{fig:imagenet}Comparing the explanations given by LIME  \emph{vs} approximate explanations obtained by summing the integrated gradient over the LIME superpixels. Here we explain the top predicted class for images of the ILSVRC2017 test data with the InceptionV3 network. In both cases, we showcase the top five positive coefficients. Qualitatively, the explanations obtained are quite similar, identifying close superpixels when they are not matching exactly.}
\end{figure*}

% put the formula for IG
For both these reasons, we build a linear approximation of~$f$ between $\xi$ and $\xibar$ using the \emph{averaged gradients on a linear path} between $\xi$ and $\xibar$.
Formally, we define
\begin{equation}
\label{eq:def-integrated-gradient}
\ig_u \defeq \int_{0}^1 \frac{\partial f ((1-\alpha)\xi + \alpha\xibar)}{\partial x_u}\Diff \alpha
\end{equation}
the averaged gradient at pixel $u$.
We approximate this integral by a Riemann sum, that is,
\begin{equation}
\label{eq:def-approx-ig}
\igapp_u \defeq \frac{1}{m}\sum_{k=1}^{m} \frac{\partial f ((1-\frac{k}{m})\xi + \frac{k}{m}\xibar)}{\partial x_u}
\, .
\end{equation}
Subsequently, we approximate $f(x)$ by
$
(x-\xibar)^\top \igapp + f(\xibar)
%\, .
$.
Applying Proposition~\ref{prop:beta-computation-linear} to this approximation we obtain the approximate explanations
\begin{equation}
\label{eq:approx-explanation}
\forall 1\leq j\leq d, \quad \betaapprox_j = \sum_{u\in J_j} (\xi_u-\xibar_u)\cdot \igapp_u
\, .
\end{equation}
Inside the sum, we recognize the definition of \emph{integrated gradients} between $\xi$ and $\xibar$ \citep{sundararajan2017axiomatic}, another interpretability method. 
Eq.~\eqref{eq:approx-explanation} therefore corresponds to \textbf{the sum of integrated gradients over superpixel $j$}.

%%%%%%%%%%%%%%%%%%%%%%%%%%%%%%%%%%%%%%%%%%%%%%%%%%%%%%

\subsection{Experiments}
\label{sec:experiments}

In this section, we show experimentally that LIME explanations are similar to the approximated explanations derived in the previous section. 
The code of all experiments is available at \url{https://github.com/dgarreau/image_lime_theory}

\textbf{Setting.}
We first considered images from the CIFAR10 dataset~\citep{krizhevsky2009learning}, that is, $32\times 32$ RGB images belonging to ten categories. 
For a subset of $1000$ images of the test set, we computed the explanations given by LIME with default settings, with the exception of the kernel size used by the quickshift algorithm which we decreased to~$1$ to get wider superpixels. 
We compared these explanations with the approximated explanations of Section~\ref{sec:approx} for four different models. 
First, we started with a very simple one-hidden-layer neural network, trained to $35\%$ accuracy. 
We then moved to VGG-like architectures \citep{simonyan2014very}, progressively increasing the number of blocks in the model (from $1$ to $3$). 
For each model, we considered the function corresponding to the most likely class for $\xi$. 
%, with the idea that it was the most interesting function to consider. 
We then collected the indices of the superpixels associated to the top five and top ten positive average coefficients. 
For the sum of integrated gradients, we considered $m=20$ steps in Eq.~\eqref{eq:def-approx-ig} as in \citet{sundararajan2017axiomatic}. 
The results are presented in Table~\ref{tab:cifar}.

We then moved to more realistic images coming from the test set of the 2017 large scale visual recognition challenge \citep[LSVRC,][]{russakovsky2015imagenet}. 
We used three pretrained models from the Keras framework: MobileNetV2~\citep{sandler2018mobilenetv2}, DenseNet121~\citep{huang2017densely}, and InceptionV3~\citep{szegedy2016rethinking} with default input shape  $(299,299,3)$. 
Again, we compared the LIME default explanations to the approximated explanations for $1000$ of these images. 
Qualitative results are presented in Figure~\ref{fig:imagenet} for the InceptionV3 network, while Table~\ref{tab:imagenet} contains the quantitative results. 

\textbf{Metric.}
We compared the list of the top $5$ and $10$ positive coefficients \emph{via} the \emph{Jaccard sindex} (also know as Jaccard similarity coefficient), that is, the size of the intersection divided by the size of the union of the two lists.
Hence a Jaccard index of $1$ means perfect match between the identified superpixels whereas a Jaccard index of $0$ complete disagreement. 
Note that the Jaccard similarity between a fixed set of size $5$ (resp. $10$) and a random subset of $\{1,\ldots,60\}$ is equal to $0.05$ (resp. $0.06$). 
Here, $60$ is the observed average number of superpixels produced by the quickshift algorithms for the images at hand. 

\begin{table}[t]
	\caption{\label{tab:cifar}Comparison between LIME and approximated explanations for CIFAR-10. 
		For each model, we report JX, the Jaccard similarity between the top $X$ positive coefficients.}
	\vskip 0.15in
	\ra{1.0}
	\centering
	\begin{tabular}{@{}cccccc@{}} \toprule
		Model & \# param. & \# layers & acc. & J5 & J10 \\ \midrule
		$1$-layer & 330K & 1 & 0.35 & 0.99 & 0.99 \\
		VGG1 & 1M & 4 & 0.67 & 0.81 & 0.85 \\ 
		VGG2 & 600K & 8 & 0.69 & 0.75 & 0.81 \\
		VGG3 & 550K & 12 & 0.70 & 0.71 & 0.76 \\
		\bottomrule
	\end{tabular}
	\vskip -0.1in
\end{table}

\textbf{Results.} 
Without being a perfect match, we observe a \textbf{substantial overlap between the LIME explanations and the approximated explanations} for all the models and datasets that we tried. 
This is particularly striking for simple models. 
More precisely, the Jaccard similarities observed are several times higher than what a random guess would produce. 
This is surprising since we are considering a linear approximation of highly non-linear functions. 
As a matter of fact, the exact values of the interpretable coefficients are quite different. 
Nevertheless, they are sufficiently close so that the sets of superpixels identified by both methods are consistently overlapping.

We notice that this link seems to weaken when the models become too complex, while still a third of identified superpixels are common for InceptionV3. 
However, visual inspection reveals that the superpixels identified by both methods remain close from each other even when they are distinct (see Figure~\ref{fig:imagenet} and additional qualitative results in appendix).

% limits
We want to emphasize that, if the model is not smooth, the link between approximate explanations in the sense of Eq.~\eqref{eq:approx-explanation} and LIME does not exist anymore. 
For instance, a random forest model based on CART trees has gradient equal to zero everywhere. 
Therefore, the integrated gradient is also zero, and $\betaapprox_j=0$ for any $j$. 
We also want to point out that we did not evaluate $\betaapprox$ as as an interpretability method. 
In particular, it could be the case that the associated explanations are of a lesser quality than LIME's.

\begin{table}[t]
	\caption{\label{tab:imagenet}Comparison between LIME and approximated explanations for LSVRC images. 
	}
	\vskip 0.15in
	\ra{1.0}
	\centering
	\begin{tabular}{@{}cccccc@{}} \toprule
		Model & \# param. & \# layers & acc. & J5 & J10 \\ \midrule
		MobileNetV2 & 3.5M & 88 & 0.90 & 0.43 & 0.54 \\
		DenseNet121 & 8.0M & 121 & 0.92 & 0.42 & 0.44 \\ 
		InceptionV3 & 23.9M & 159 & 0.94 & 0.35 & 0.36 \\
		\bottomrule
	\end{tabular}
	\vskip -0.1in
\end{table}

\paragraph{Computation time.}
Setting aside the segmentation step, each run of LIME requires $n=1000$ queries of the model, whereas the averaged gradient estimation requires $m=20$ queries. 
In the favorable scenario where getting a gradient is as costly as a model query, computing the approximated explanations is much faster than LIME. 

%%%%%%%%%%%%%%%%%%%%%%%%%%%%%%%%%%%%%%%%%%%%%%%%%%%%%%

\section{Conclusion}

In this paper, we proposed the first theoretical analysis of LIME for images. 
We showed that the explanations provided make sense for elementary shape detectors and linear models. 
As a consequence of this analysis, we discovered that for smooth models the interpretable coefficients of LIME for images resemble to the sum of integrated gradients over the LIME superpixels. 

As future work, we plan on tackling more complex models. 
A starting point is the study of polynomial functions: obtaining a statement analogous to Proposition~\ref{prop:beta-computation-linear} would open the door to more precise expression for the limit explanation depending on the higher derivatives of $f$.

%%%%%%%%%%%%%%%%%%%%%%%%%%%%%%%%%%%%%%%%%%%%%%%%%%%%%%%%%%%%

\section*{Acknowledgments}

This work was partly funded by the UCA DEP grant. 
The authors want to thank Ulrike von Luxburg for her insights in the writing phase of the paper.  

\bibliographystyle{abbrvnat}
\bibliography{biblio}

\newpage
% !TeX spellcheck = en_US

\onecolumn

\icmltitle{Supplementary material for the paper: \\ ``What does LIME really see in images?''}

\thispagestyle{empty}

\section*{Organization of the supplementary material}

In this appendix, we present the detailed proof of our main results (Theorem~1 and Proposition~2) and additional qualitative results. 
We follow the proof scheme of \citet{garreau_luxburg_2020_arxiv}. 
In a nutshell, when $\lambda=0$, the main problem
\begin{equation}
\label{eq:main-problem}
\betahat_n^{\lambda} \in \argmin{\beta\in\Reals^{d+1} } \biggl\{ \sum_{i=1}^{n} \pi_i(y_i-\beta^\top z_i)^2 + \lambda \norm{\beta}^2 \biggr\}
\end{equation}
reduces to least squares, with $\betahat_n$ given in closed-form by
\[
\betahat_n = (Z^\top WZ)^{-1}Z^\top Wy
\, ,
\]
with $Z\in\{0,1\}^{n\times d}$ the matrix whose lines are given by the $z_i$s and $W$ the diagonal matrix such that $W_{i,i}=\pi_i$. 
Setting $\Sigmahat_n \defeq \frac{1}{n}Z^\top WZ$ and $\Gammahat_n\defeq \frac{1}{n}Z^\top Wy$, the study of $\betahat_n$ can be split in two parts: the examination of $\Sigmahat_n$ (Section~\ref{sec:sigma}), and then that of $\Gammahat_n$ (Section~\ref{sec:gamma}). 
We put everything together in Section~\ref{sec:study-of-beta}, proving the concentration of $\betahat_n$ and providing the expression of $\beta^f$.
All technical results are collected in Section~\ref{sec:technical}. 
Finally, additional qualitative results are presented in Section~\ref{sec:experiments}. 

%%%%%%%%%%%%%%%%%%%%%%%%%%%%%%%%%%%%%%%%%%%%%%%%%%%%%%%%%%%%%%%%%%%

\section{Study of $\Sigmahat_n$}
\label{sec:sigma}

We start by the study of $\Sigmahat_n$, first computing its limit $\Sigma$ when $n\to +\infty$ (Section~\ref{sec:sigma-computation}). 
We show that $\Sigma$ is invertible in closed-form in Section~\ref{sec:sigma-inverse}. 
We then proceed to show that $\Sigmahat_n$ is concentrated around $\Sigma$ in Section~\ref{sec:concentration-sigmahat}. 
We conclude this section by obtaining a control on the operator norm of $\Sigma^{-1}$ (Section~\ref{sec:control-opnorm}), a technical requirement for the proof of the main result. 

\subsection{Computation of $\Sigma$}
\label{sec:sigma-computation}

By definition of $Z$ and $W$, the matrix $\Sigmahat_n$ can be written
\[
\Sigmahat =
\begin{pmatrix}
\frac{1}{n}\sum_{i=1}^n \pi_i & \frac{1}{n}\sum_{i=1}^n \pi_i z_{i,1} & \cdots & \frac{1}{n}\sum_{i=1}^n \pi_i z_{i,d} \\ 
\frac{1}{n}\sum_{i=1}^n \pi_i z_{i,1}  & \frac{1}{n}\sum_{i=1}^n \pi_i z_{i,1}  & \cdots &  \frac{1}{n}\sum_{i=1}^n \pi_i z_{i,1}z_{i,d} \\ 
\vdots & \vdots  & \ddots & \vdots \\ 
\frac{1}{n}\sum_{i=1}^n \pi_i z_{i,d} & \frac{1}{n}\sum_{i=1}^n \pi_i z_{i,1}z_{i,d} & \cdots & \frac{1}{n}\sum_{i=1}^n \pi_i z_{i,d}
\end{pmatrix}
\in\Reals^{(d+1)\times (d+1)}
\, .
\]

Recall that we defined the random variable $z$ such that $z_i$ is i.i.d. $z$ for any $i$, as well as $\pi$ and $x$ the associated weights and perturbed samples. 
For any $p\geq 0$, we also defined $\alpha_p = \expec{\pi\prod_{i=1}^p z_i}$ (Definition~1). 
Taking the expectation with respect to $z$ in the previous display, we obtain
\[
\Sigma_{j,k} = 
\begin{cases}
\alpha_0 &\text{ if } j=k=0, \\
\alpha_1 &\text{ if } j=0 \text{ and } k> 0 \text{ or } j> 0 \text{ and } k=0 \text{ or } j=k> 0, \\
\alpha_2 &\text{ otherwise. }
\end{cases}
\]

As promised, it is possible to compute the $\alpha$ coefficients in closed-form. 
Let us denote by $S$ the number of superpixel deletions. 
Since the coordinates of $z$ are i.i.d. Bernoulli with parameter $1/2$, we deduce that $S$ is a \emph{binomial} random variable of parameters $d$ and $1/2$. 
Note that, conditionally to $S=s$, $\sum_j z_j = d-s$ and therefore $\pi=\psi(s/d)$ with
\begin{equation}
\label{eq:def-psi}
\forall t \in [0,1],\quad \psi(t) \defeq \exp{\frac{-(1-\sqrt{1-t})^2}{2\nu^2}}
\end{equation}
as in the paper. 
As a consequence of these observations, we have:

\begin{proposition}[Computation of the $\alpha$ coefficients]
	\label{prop:computation-alpha-coefficients}
	Let $p\geq 0$ be an integer. 
	Then
	\[
	\alpha_p = \frac{1}{2^d}\sum_{s=0}^d \binom{d-p}{s}\psi(s/d)
	\, .
	\]
\end{proposition}

\begin{proof}
	We write
\begin{align*}
\alpha_p &= \expec{\pi z_1 \cdots z_p} \\
&= \sum_{s=0}^d \expecunder{\pi z_1 \cdots z_p}{s} \proba{S=s} \tag{law of total expectation} \\
&= \frac{1}{2^d}\sum_{s=0}^d \binom{d}{s} \condexpecunder{\pi z_1 \cdots z_p}{z_1=1,\ldots,z_p=1}{s}\probaunder{z_1=1,\ldots,z_p=1}{s}  \tag{$S\sim\binomial{n,1/2}$} \\
&= \frac{1}{2^d}\sum_{s=0}^d \binom{d}{s} \psi(s/d) \probaunder{z_1=1,\ldots,z_p=1}{s} \tag{definition of $\psi$} \\
\alpha_p &= \frac{1}{2^d}\sum_{s=0}^d \binom{d}{s} \frac{(d-p)!}{d!}\cdot \frac{(d-s)!}{(d-s-p)!} \psi(s/d) \tag{Lemma~\ref{lemma:basic-computations}}
\end{align*}
We conclude by some algebra. 
\end{proof}

It is quite straightforward to compute the limits of the $\alpha$ coefficients when $\nu\to +\infty$. 
In fact, since $\exps{-1/(2\nu^2)}\leq \psi(t) \leq 1$ for any $\nu > 0$, we have the following bounds on $\alpha_p$:

\begin{lemma}[Bounding the $\alpha$ coefficients]
\label{lemma:alpha-bound}
For any $p\geq 0$, we have
\[
\frac{\exps{\frac{-1}{2\nu^2}}}{2^p} \leq \alpha_p \leq \frac{1}{2^p}
\, .
\]
In particular, when $\nu\to +\infty$, we have $\alpha_p\to \frac{1}{2^p}$ for any $p\geq 0$. 
\end{lemma}

We demonstrate these approximations in Figure~\ref{fig:alpha-bounds}.

\begin{figure}
\begin{center}
\includegraphics[scale=0.3]{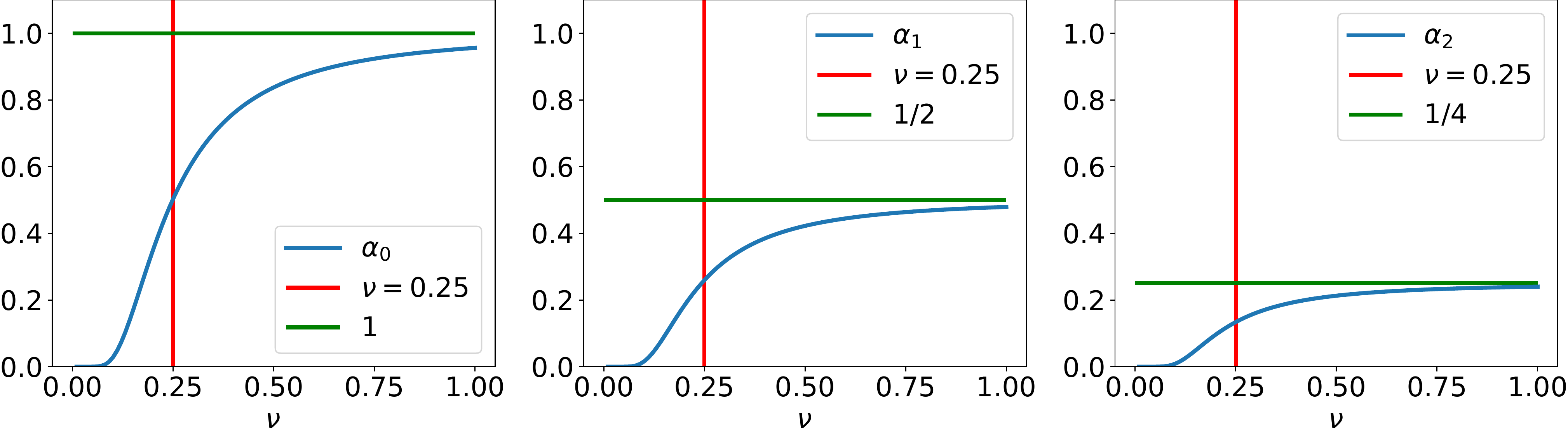}
\end{center}
\caption{\label{fig:alpha-bounds}The first three $\alpha$ coefficients as a function of the bandwidth $\nu$ for $d=50$. In green the limit value given by Lemma~\ref{lemma:alpha-bound}. }
\end{figure}

%%%%%%%%%%%%%%%%%%%%%%%%%%%%%%%%%%%%%%%%%%%%%%%%%%%%%%%%%%%%%%%%%%%

\subsection{$\sigma$ coefficients}
\label{sec:sigma-inverse}

Since the structure of $\Sigma$ is the same as in the text case \citep{mardaoui2020analysis}, we can invert it similarly. 

\begin{proposition}[Inverse of $\Sigma$]
	\label{prop:inverse-sigma}
	For any $d\geq 1$, recall that we defined \[
	\begin{cases}
	\sigma_1 &= -\alpha_1
	\, , \\
	\sigma_2 &= \frac{(d-2)\alpha_0 \alpha_2 - (d-1)\alpha_1^2 + \alpha_0\alpha_1}{\alpha_1-\alpha_2}\, , \\
	\sigma_3 &= \frac{\alpha_1^2-\alpha_0\alpha_2}{\alpha_1-\alpha_2 }
	\, ,
	\end{cases}
	\]
	and $\dencst_d = (d-1)\alpha_0\alpha_2 -d\alpha_1^2 + \alpha_0\alpha_1$. 
Let us further define $\sigma_0 \defeq (d-1)\alpha_2 + \alpha_1$. 
	Assume that $\dencst_d\neq 0$ and $\alpha_1\neq \alpha_2$. 
	Then $\Sigma$ is invertible, and it holds that
	\begin{equation}
	\label{eq:sigma-inverse-computation}
	\Sigma^{-1} = 
	\frac{1}{\dencst_d}
	\begin{pmatrix}
	\sigma_0 &  \sigma_1 & \sigma_1  &\cdots  & \sigma_1  \\ 
	
	\sigma_1 & \sigma_2 & \sigma_3 & \cdots  &  \sigma_3 \\
	
	\sigma_1 &  \sigma_3 & \sigma_2   &  \ddots  &  \vdots  \\ 
	
	\vdots  &  \vdots &  \ddots &  \ddots  & \sigma_3 \\
	
	\sigma_1 & \sigma_3 & \cdots & \sigma_3 &  \sigma_2  \\
	\end{pmatrix}
	\in\Reals^{(d+1)\times (d+1)}
	\, .
	\end{equation}
\end{proposition}

\begin{figure}
\begin{center}
\includegraphics[scale=0.285]{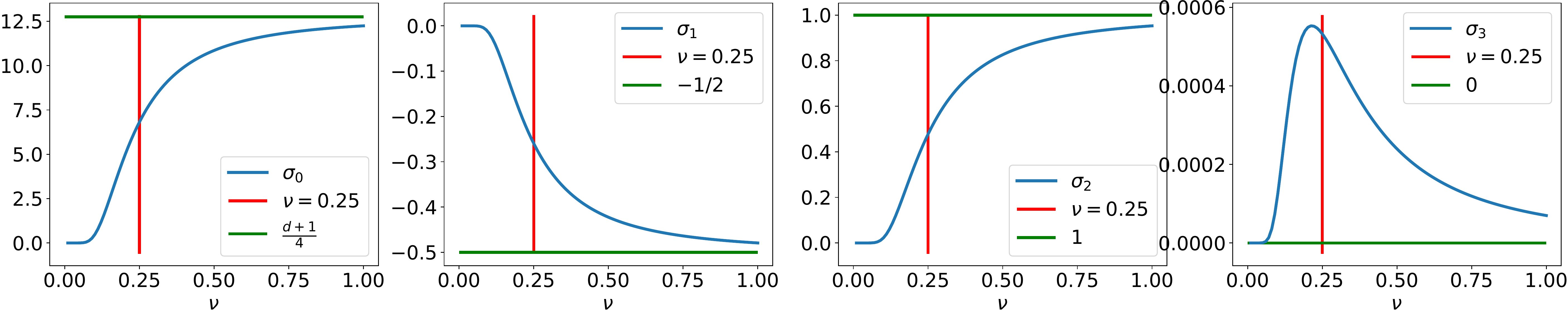}
\end{center}
\caption{\label{fig:sigma-bounds}The first four $\sigma$ coefficients as a function of the bandwidth $\nu$ for $d=50$. In green, the limit values given by Eq.~\eqref{eq:lim-sigmas}.}
\end{figure}
\noindent
\begin{minipage}{0.7\textwidth}
From Lemma~\ref{lemma:alpha-bound}, we deduce
\begin{equation}
\label{eq:lim-sigmas}
\sigma_0 \to \frac{d+1}{4}\, , \quad
\sigma_1 \to \frac{-1}{2}\, , \quad
\sigma_2 \to 1 \, , \quad
\sigma_3 \to 0 \, , \quad\text{ and }\quad 
\dencst_d \to 1/4
\, .
\end{equation}
when $\nu\to +\infty$. 
We illustrate this in Figure~\ref{fig:sigma-bounds}. 
Now, Proposition~\ref{prop:inverse-sigma} requires $\alpha_1\neq \alpha_2$ and $\dencst_d\neq 0$ in order for $\Sigma$ to be invertible. 
One of the consequences of the following result is that these conditions are always satisfied. 

\begin{proposition}[$\Sigma$ is invertible]
\label{prop:inv-lower-bound}
	Let $d\geq 1$ and $\nu >0$. 
	Then $\alpha_1-\alpha_2\geq \exps{\frac{-1}{2\nu^2}}/4$ and $\dencst_d \geq \exps{\frac{-1}{\nu^2}}/4$.
\end{proposition}

Note that in this case the lower bound obtained on $\dencst_d$ is tight. 
We show the evolution of $\dencst_d$ with respect to the bandwidth in Figure~\ref{fig:dencst-bounds}. 
\end{minipage}
\hfill
\begin{minipage}{0.2\textwidth}
	\begin{center}
		\includegraphics[scale=0.3]{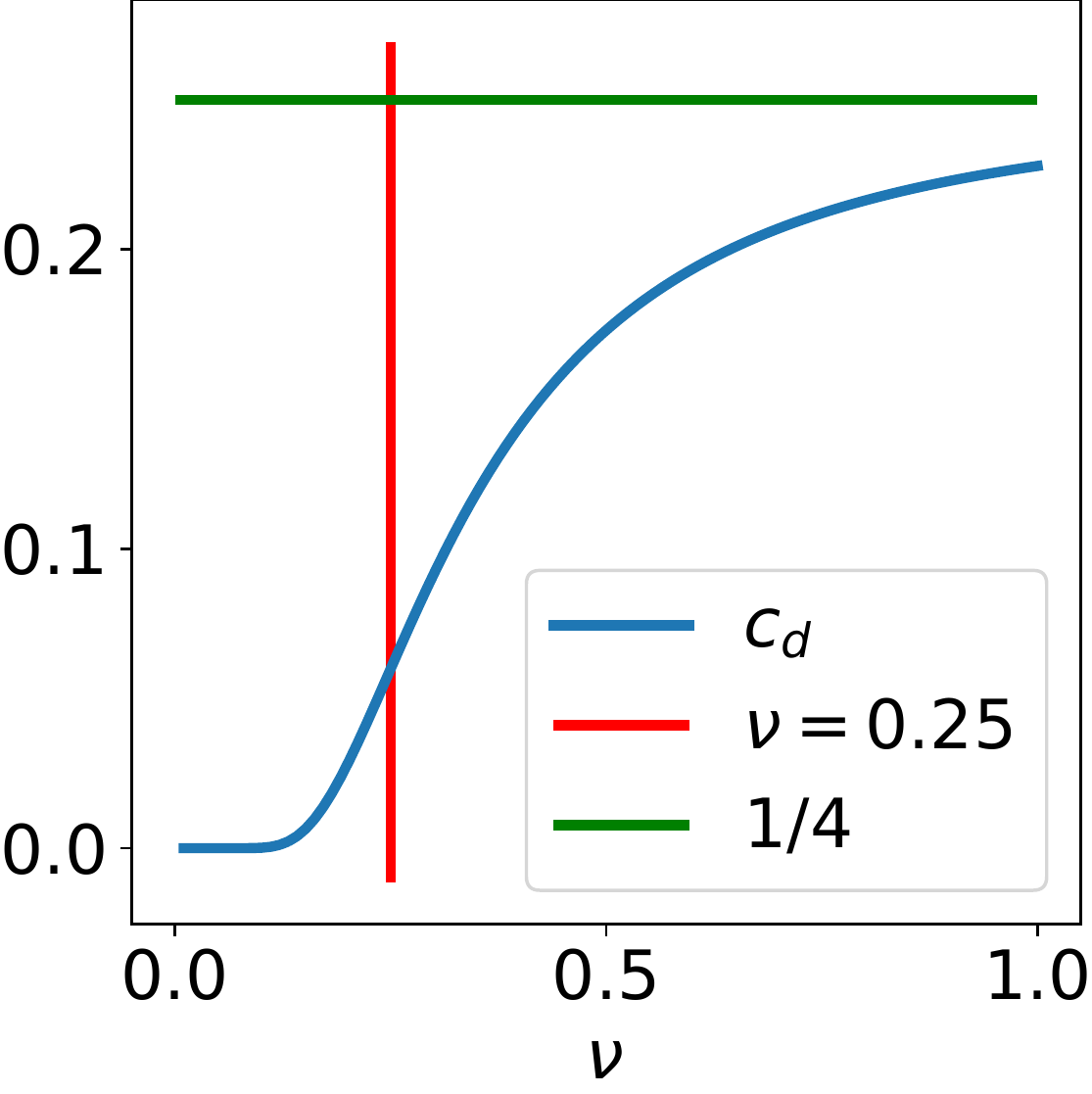}
	\end{center}
	\captionof{figure}{\label{fig:dencst-bounds}Evolution of $\dencst_d$ with respect to $\nu$ when $d=50$. }
\end{minipage}

\begin{proof}
	By definition of the $\alpha$ coefficients and Pascal identity, it holds that
	\begin{equation}
	\label{eq:diff-alpha}
	\alpha_p - \alpha_{p+1} = \frac{1}{2^d} \sum_{s=0}^d \binom{d-p-1}{s-1}\psi\left(\frac{s}{d}\right) 
	\, ,
	\end{equation}
	for any $p\geq 0$.
	Since $\exps{-1/(2\nu^2)}\leq \psi(t)\leq 1$ for any $1\leq t\leq 1$, we deduce from Eq.~\eqref{eq:diff-alpha} that, for any $p\geq 0$, 
\begin{equation}
\label{eq:bound-diff-alpha-coefficients}
\frac{\exps{\frac{-1}{2\nu^2}}}{2^{p+1}} \leq \alpha_p-\alpha_{p+1}\leq \frac{1}{2^{p+1}}
\, .
\end{equation}
We deduce the lower bound on $\alpha_1-\alpha_2$ by setting $p=1$ in the previous display.

Let us turn to $\dencst_d$. 
We write 
	\begin{align*}
	\dencst_d &= d\alpha_1(\alpha_0-\alpha_1) - (d-1) \alpha_0 (\alpha_1-\alpha_2) \\
	&= \frac{1}{4^d}\left[ d\cdot \sum_{s=0}^d \binom{d-1}{s}\psi\left(\frac{s}{d}\right) \cdot \sum_{s=0}^d \binom{d-1}{s-1}\psi\left(\frac{s}{d}\right) - (d-1)\cdot \sum_{s=0}^d \binom{d}{s}\psi\left(\frac{s}{d}\right) \cdot \sum_{s=0}^d \binom{d-2}{s-1}\psi\left(\frac{s}{d}\right)  \right] \tag{using Eq.~\eqref{eq:diff-alpha}} \\
	\dencst_d &= \frac{1}{4^d}\left[ \sum_{s=0}^d \binom{d-1}{s}\psi\left(\frac{s}{d}\right) \cdot \sum_{s=0}^d s\binom{d}{s}\psi\left(\frac{s}{d}\right) - \sum_{s=0}^d \binom{d}{s}\psi\left(\frac{s}{d}\right) \cdot \sum_{s=0}^d s\binom{d-1}{s}\psi\left(\frac{s}{d}\right) \right]
	\, ,
	\end{align*}
	where we used elementary properties of the binomial coefficients in the last display. 
	For any $0\leq s\leq d$, let us set 
	\[
	A_s \defeq \binom{d-1}{s}\sqrt{\psi\left(\frac{s}{d}\right)}\, , 
	B_s \defeq s\sqrt{\psi\left(\frac{s}{d}\right)}\, ,
	C_s \defeq \sqrt{\psi\left(\frac{s}{d}\right)}\, ,
	\text{ and }
	D_s \defeq \binom{d}{s} \sqrt{\psi\left(\frac{s}{d}\right)}
	\, .
	\]
	With these notation,
	\[
	\dencst_d = \frac{1}{4^d}\left[ \sum_s A_s C_s \cdot \sum_s B_sD_s - \sum_s A_sB_s \cdot \sum C_s D_s \right]
	\, .
	\]
	According to the four-letter identity (Proposition~\ref{prop:four-letter}), we can rewrite $\dencst_d$ as
	\begin{align*}
	\dencst_d &= \frac{1}{4^d}\sum_{s<t} (A_sD_t-A_tD_s)(C_sB_t-C_tB_s) \\
	&= \frac{1}{4^d} \sum_{s<t} (t-s)\left(\binom{d-1}{s}\binom{d}{t} - \binom{d-1}{t}\binom{d}{s}\right)\psi\left(\frac{s}{d}\right)\psi\left(\frac{t}{d}\right) \\
	\dencst_d &= \frac{1}{d\cdot 4^d} \sum_{s<t} \binom{d}{s}\binom{d}{t}(s-t)^2 \psi\left(\frac{s}{d}\right)\psi\left(\frac{t}{d}\right)
	\, .
	\end{align*}
	Since  $\exps{-1/(2\nu^2)}\leq \psi(t)\leq 1$ for any $1\leq t\leq 1$, all that is left to do is to control the double sum. 
	According to Proposition~\ref{prop:combinatorial-identity}, we have 
	\[
	\sum_{s<t} \binom{d}{s}\binom{d}{t}(s-t)^2 = d\cdot 4^{d-1}
	\, .
	\]
We deduce that
\begin{equation}
\label{eq:control-dencst}
\frac{\exps{\frac{-1}{2\nu^2}}}{4} \leq \dencst_d \leq \frac{1}{4}
\, .
\end{equation}
\end{proof}

We conclude this section with useful relationships between $\alpha$ and $\sigma$ coefficients. 

\begin{proposition}[Useful equalities]
	\label{prop:cancellation-equalities}
	Let $\alpha_p$, $\sigma_p$, and $\dencst_d$ be defined as above. 
	Then it holds that
	\begin{equation}
	\label{eq:aux-rel-0}
	\sigma_0\alpha_1 + \sigma_1\alpha_1 + (d-1)\sigma_1\alpha_2 = 0
	\, ,
	\end{equation}
	\begin{equation}
	\label{eq:aux-rel-1}
	\sigma_1\alpha_1 + \sigma_2\alpha_1+(d-1)\sigma_3\alpha_2 = \dencst_d
	\, ,
	\end{equation}
	\begin{equation}
	\label{eq:aux-rel-2}
	\sigma_1\alpha_1 + \sigma_2\alpha_2+\sigma_3\alpha_1+(d-2)\sigma_3\alpha_2 = 0
	\, ,
	\end{equation}
	\begin{equation}
	\label{eq:aux-rel-3}
	\sigma_1\alpha_0 + \sigma_2\alpha_1+(d-1)\sigma_3\alpha_1 = 0
	\, ,
	\end{equation}
\begin{equation}
\label{eq:aux-rel-4}
\sigma_0\alpha_0 + d\sigma_1\alpha_1 = \dencst_d
\, .
\end{equation}
\end{proposition}

\begin{proof}
	Straightforward from the definitions. 
\end{proof}

%%%%%%%%%%%%%%%%%%%%%%%%%%%%%%%%%%%%%%%%%%%%%%%%%%%%%%%%%%%%%%%%%%%

\subsection{Concentration of $\Sigmahat_n$}
\label{sec:concentration-sigmahat}

% everything is the same since we are bounded 
We now turn to the concentration of $\Sigmahat_n$ around $\Sigma$. 
More precisely, we show that $\Sigmahat_n$ is close to $\Sigma$ in operator norm, with high probability. 
Since the definition of $\Sigmahat_n$ is identical to the one in the Tabular LIME case, we can use the proof machinery of \citet{garreau_luxburg_2020_arxiv}. 

\begin{proposition}[Concentration of $\Sigmahat_n$]
\label{prop:sigmahat-concentration}
For any $t\geq 0$, 
\[
\proba{\opnorm{\Sigmahat_n - \Sigma} \geq t} \leq 4d\cdot \exp{\frac{-nt^2}{32d^2}}
\, .
\]
\end{proposition}

\begin{proof}
	We can write $\Sigmahat=\frac{1}{n}\sum_i \pi_i Z_iZ_i^\top$. 
	The summands are bounded i.i.d. random variables, thus we can apply the matrix version of Hoeffding inequality. 
	More precisely, the entries of $\Sigmahat_n$ belong to $[0,1]$ by construction, and Lemma~\ref{lemma:alpha-bound} guarantees that the entries of $\Sigma$ also belong to $[0,1]$. 
	Therefore, if we set $M_i\defeq \frac{1}{n}\pi_i Z_iZ_i^\top -\Sigma$, then the $M_i$ satisfy the assumptions of Theorem~21 in \citet{garreau_luxburg_2020_arxiv} and we can conclude since $\frac{1}{n}\sum_i M_i = \Sigmahat_n-\Sigma$. 
\end{proof}

%%%%%%%%%%%%%%%%%%%%%%%%%%%%%%%%%%%%%%%%%%%%%%%%%%%%%%%%%%%%%%%%%%%%%

\subsection{Control of $\opnorm{\Sigma^{-1}}$}
\label{sec:control-opnorm}

In this section, we obtain a control on the operator norm of the inverse covariance matrix. 
Our strategy is to bound the norm of the $\sigma$ coefficients. 
We start with the control of $\alpha_1^2-\alpha_0\alpha_2$, a quantity appearing in $\sigma_2$ and $\sigma_3$. 

\begin{lemma}[Control of $\alpha_1^2-\alpha_0\alpha_2$]
\label{lemma:control-det}
For any $d\geq 2$, we have 
\[
\abs{ \alpha_1^2-\alpha_0\alpha_2 }\leq \frac{1}{2d}
\, .
\]
\end{lemma}

\begin{proof}
By definition of the $\alpha$ coefficients, we know that
\[
\alpha_1^2-\alpha_0\alpha_2 = \frac{1}{4^d} \left[ \left(\sum_{s=0}^{d}\binom{d-1}{s}\psi\left(\frac{s}{d}\right)\right)^2 - \left(\sum_{s=0}^{d}\binom{d}{s}\psi\left(\frac{s}{d}\right) \right)\cdot \left(\sum_{s=0}^{d} \binom{d-2}{s}\psi \left(\frac{s}{d}\right)\right)\right]
\, .
\]
Let us ignore the $1/4^d$ normalization for now, and set $w_s\defeq \binom{d}{s}\psi\left(\frac{s}{d}\right)$. 
Elementary manipulations of the binomial coefficients allow us to rewrite the previous display as
\begin{equation}
\label{eq:aux-bound-det}
\left(\sum_{s=0}^{d} \frac{d-s}{d}w_s\right)^2 - \left(\sum_{s=0}^{d} w_s \right) \cdot \left(\sum_{s=0}^{d}\frac{d-s}{d}\cdot \frac{d-s-1}{d-1} w_s \right)
\, .
\end{equation}
Let us notice that 
\[
\frac{d-s}{d} - \frac{d-s-1}{d-1} = \frac{s}{d(d-1)}
\, .
\]
Thus we can split Eq.~\eqref{eq:aux-bound-det} in two parts. 

The first part is reminiscent of the Cauchy-Schwarz-like expression that appears in the proof of Proposition~\ref{prop:inv-lower-bound}:
\begin{equation}
\label{eq:aux-bound-det-2}
\left(\sum_{s=0}^{d} \frac{d-s}{d}w_s\right)^2 - \left(\sum_{s=0}^{d} w_s \right) \cdot \left(\sum_{s=0}^{d}\frac{(d-s)^2}{d^2} w_s \right)
\, .
\end{equation}
We use, again, the four letter identity (Proposition~\ref{prop:four-letter}) to bound this term. 
Namely, for any $0\leq s\leq d$, let us set 
\[
A_s =B_s \defeq \frac{d-s}{d}\sqrt{w_s}\, , \quad\text{ and }\quad C_s=D_s \defeq \sqrt{w_s}
\, .
\]
Then we can rewrite Eq.~\eqref{eq:aux-bound-det-2} as 
\begin{equation}
\label{eq:aux-bound-det-3}
\sum_{s<t} (A_sD_t-A_tD_s)(C_sB_t-C_tB_s) = \frac{-1}{d^2}\sum_{s<t} (t-s)^2\binom{d}{s}\binom{d}{t}\psi\left(\frac{s}{d}\right)\psi\left(\frac{t}{d}\right)
\, .
\end{equation}
According to Proposition~\ref{prop:combinatorial-identity}, Eq.~\eqref{eq:aux-bound-det-3} is bounded by $d\cdot 4^{d-1}/d^2=4^{d-1}/d$. 

The second part of Eq.~\eqref{eq:aux-bound-det} reads
\[
\left(\sum_{s=0}^{d}w_s\right)\cdot \left(\sum_{s=0}^{d} \frac{d-s}{d}\cdot \frac{s}{d(d-1)}w_s \right)
\, .
\]
Since $\psi$ is bounded by $1$, coming back to the definition of the $w_s$, it is straightforward to show that $\abs{\sum_s w_s}\leq 2^{d}$ and that $\abs{\sum_s s(d-s)w_s}\leq d(d-1)2^{d-2}$. 
We deduce that (the absolute value of) this second term is upper bounded by $4^{d-1}/d$.

Putting together the bounds obtained on both terms and renormalizing by $4^d$, we obtain that 
\[
\abs{\alpha_1^2-\alpha_0\alpha_2} \leq \frac{1}{4^d}\left[\frac{4^{d-1}}{d}+\frac{4^{d-1}}{d}\right] = \frac{1}{2d}
\, .
\]
\end{proof}

We now have everything we need to provide reasonably tight upper bounds for the $\sigma$ coefficients. 

\begin{proposition}[Bounding the $\sigma$ coefficients]
\label{prop:bound-sigma-coefficients}
Let $d\geq 2$. 
Then the following holds:
\[
\abs{\sigma_0} \leq \frac{3d}{4} \, ,\quad \abs{\sigma_1}\leq \frac{1}{2}\, ,\quad \abs{\sigma_2} \leq 2\exps{\frac{1}{2\nu^2}}\, ,\quad \text{ and }\quad \abs{\sigma_3} \leq \frac{2\exps{\frac{1}{2\nu^2}}}{d}
\, .
\]
\end{proposition}

\begin{proof}
From Lemma~\ref{lemma:alpha-bound} and the definition of $\sigma_0$, we have
\[
\abs{\sigma_0} = \abs{(d-1)\alpha_2 + \alpha_1} \leq \frac{d-1}{4} + \frac{1}{2} = \frac{d+3}{4}
\, .
\]
We deduce the first result since $d\geq 2$. 
Next, since $\sigma_1=-\alpha_1$, we obtain $\abs{\sigma_1} \leq 1/2$ directly from Lemma~\ref{lemma:alpha-bound}. 
Regarding the last two coefficients, recall that Proposition~\ref{prop:inv-lower-bound} guarantees that their common denominator $\alpha_1-\alpha_2$ is lower bounded by $\exps{\frac{-1}{2\nu^2}}/4$. 
Since 
\[
(d-2)\alpha_0\alpha_2 - (d-1)\alpha_1^2 + \alpha_0\alpha_1 = \dencst_d + \alpha_1^2-\alpha_0\alpha_2
\, ,
\]
we can write $\sigma_2=(\dencst_d+\alpha_1^2-\alpha_0\alpha_2)/(\alpha_1-\alpha_2)$ and deduce that 
\[
\abs{\sigma_2} \leq \frac{1/4 + 1/(2d)}{\exps{\frac{-1}{2\nu^2}}/4} \leq 2\exps{\frac{1}{2\nu^2}}
\, ,
\] 
since, according to Eq.~\eqref{eq:control-dencst}, $\dencst_d\leq 1/4$ and $\alpha_1^2-\alpha_0\alpha_2\leq 1/(2d)$ according to Lemma~\ref{lemma:control-det}. 
Finally, we write
\[
\abs{\sigma_3} = \abs{\frac{\alpha_1^2-\alpha_0\alpha_2}{\alpha_1-\alpha_2}} \leq \frac{1/(2d)}{\exps{\frac{-1}{2\nu^2}} /4} = \frac{2\exps{\frac{1}{2\nu^2}}}{d}
\, .
\]
\end{proof}

The bounds obtained in Proposition~\ref{prop:bound-sigma-coefficients} immediately translate into a control of the Frobenius norm of $\Sigma^{-1}$, which in turn yields a control over the operator norm of $\Sigma^{-1}$, as promised. 

\begin{corollary}[Control of $\opnorm{\Sigma^{-1}}$]
\label{cor:control-opnorm}
Let $d\geq 2$. 
Then $\opnorm{\Sigma^{-1}} \leq 8d\exps{\frac{1}{\nu^2}}$.
\end{corollary}

\begin{proof}
Using Proposition~\ref{prop:bound-sigma-coefficients}, we write
\begin{align*}
\frobnorm{\Sigma^{-1}}^2 &= \frac{1}{\dencst_d^2}\left[\sigma_0^2 + 2d\sigma_1^2 + d\sigma_2^2 + (d^2-d)\sigma_3^2 \right]\\
&\leq 16\exps{\frac{1}{\nu^2}}\left[\frac{9d^2}{16}+\frac{2d}{4} + 4d\exps{\frac{1}{\nu^2}} + 4\exps{\frac{1}{\nu^2}} \right] \\
&\leq 61d^2\exps{\frac{2}{\nu^2}}
\, ,
\end{align*}
where we used $d\geq 2$ in the last display. 
Since the operator norm is upper bounded by the Frobenius norm, we conclude observing that $\sqrt{61}\leq 8$. 
\end{proof}

\begin{minipage}{0.3\textwidth}
\begin{remark}
	\label{rk:tight-opnorm}
The bound on $\opnorm{\Sigma^{-1}}$ is essentially tight with respect to the dependency in $d$, as can be seen in Figure~\ref{fig:covariance-investigation}.
\end{remark}
\end{minipage}
\hfill
\begin{minipage}{0.6\textwidth}
	\begin{center}
\includegraphics[scale=0.3]{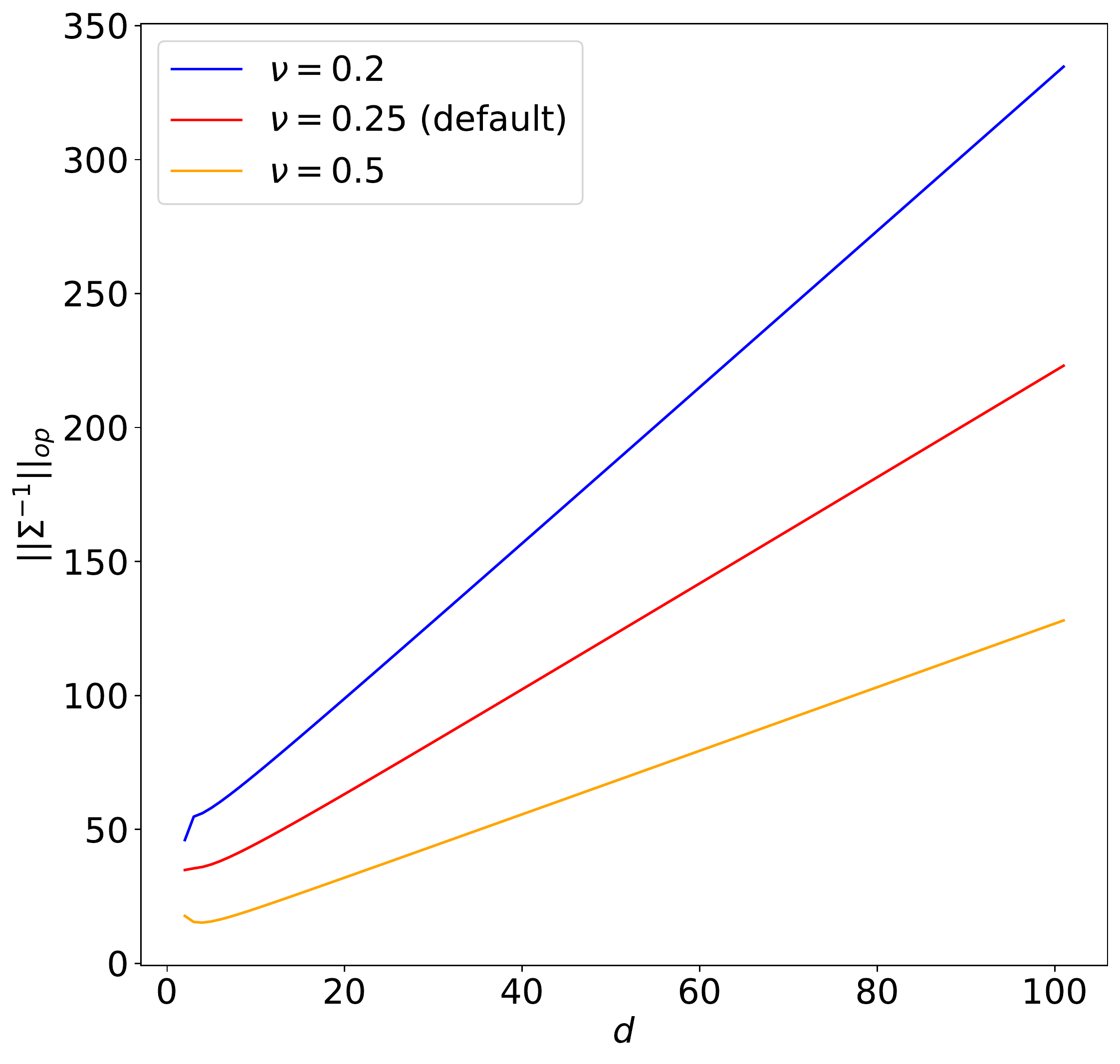}
	\end{center}
	\captionof{figure}{\label{fig:covariance-investigation}Evolution of $\opnorm{\Sigma^{-1}}$ as a function of $d$ for various values of the bandwidth parameter. The linear dependency in $d$ is striking.}
\end{minipage} 

%%%%%%%%%%%%%%%%%%%%%%%%%%%%%%%%%%%%%%%%%%%%%%%%%%%%%%%%%%%%%%%%%%%%%

\section{Study of $\Gammahat_n$}
\label{sec:gamma}

We now turn to the study of $\Gammahat_n$. 
We start by computing the limiting expression. 
Recall that we defined $\Gammahat_n=\frac{1}{n}Z^\top W y$, where $y\in\Reals^{d+1}$ is the random vector defined coordinate-wise by $y_i=f(x_i)$. 
From the definition of $\Gammahat_n$, it is straightforward that
\[ 
\Gammahat_n  =
\begin{pmatrix}
\frac{1}{n}\sum_{i=1}^n \pi_{i}f(x_i) \\ 
\frac{1}{n}\sum_{i=1}^{n} \pi_{i}{z_{i,1}}f(x_i)\\ 
\vdots\\ 
\frac{1}{n}\sum_{i=1}^{n} \pi_{i}{z_{i,d}}f(x_i) \\ 
\end{pmatrix}
\in\Reals^{d+1}
\, .
\]
As a consequence, if we define $\Gamma^f\defeq \smallexpec{\Gammahat_n}$, it holds that
\begin{equation}
\label{eq:def-gamma}
\Gamma^f = 
\begin{pmatrix}
\expec{ \pi f(x) } \\ 
\expec{\pi z_1 f(x)}\\ 
\vdots\\ 
\expec{\pi z_d f(x)} \\ 
\end{pmatrix}
\, .
\end{equation}
%Of course, Eq.~\eqref{eq:def-gamma} depends on the model $f$. 
We specialize Eq.~\eqref{eq:def-gamma} to shape detectors in Section~\ref{sec:gamma-shape} and linear models in Section~\ref{sec:gamma-linear}. 
The concentration of $\Gammahat_n$ around $\Gamma$ is obtained in Section~\ref{sec:concentration-gammahat}. 

%%%%%%%%%%%%%%%%%%%%%%%%%%%%%%%%%%%%%%%%%%%%%%%%%%%%%%%%%%%%%%%%%

\subsection{Shape detectors}
\label{sec:gamma-shape}

Recall that we defined 
\begin{equation}
\label{eq:basic-shape-detector}
\forall x\in [0,1]^D,\quad f(x) = \prod_{u\in\shape} \indic{x_u > \tau}
\, ,
\end{equation}
with $\shape = \{u_1,\ldots,u_q\}$ a fixed set of pixels indices and $\tau\in (0,1)$ a threshold. 
As in the paper, let us define $E =  \{j \text{ s.t. } J_j\cap \shape \neq \emptyset\}$ denote the set of superpixels intersecting the shape, and
\[
E_+ = \{j\in E \text{ s.t. } \xibar_j > \tau \}
\quad \text{ and } \quad 
E_- = \{j\in E \text{ s.t. } \xibar_j \leq \tau \}
\, .
\]
We also defined 
\[
\Splus = \{u\in \shape \text{ s.t. } \xi_u > \tau \}
\quad \text{ and }\quad 
\Sminus = \{u \in \shape \text{ s.t. }\xi_u \leq \tau \}
\, .
\]
In the main paper, we made the following simplifying assumption:
\begin{equation}
\label{eq:ass-simp-shape}
\forall j\in E_+,\quad J_j\cap \Sminus = \emptyset
\, .
\end{equation}
This is not the case here. 
Unfortunately, without this assumption, the expression of $\Gamma^f$ is slightly more complicated and we need to generalize the definition of the $\alpha$ coefficients. 

\begin{definition}[Generalized $\alpha$ coefficients]
For any $p,q$ such that $p+q\leq d$, we define
\begin{equation}
\label{eq:def-generalized-alpha}
\alpha_{p,q} \defeq \expec{\pi z_1 \cdots z_p\cdot (1-z_{p+1})\cdots (1-z_{p+q})}
\, .
\end{equation}
\end{definition}

We notice that, for any $1\leq p\leq d$, $\alpha_{p,0}=\alpha_p$. 
As it is the case with $\alpha$ coefficients, the generalized $\alpha$ coefficients can be computed in closed-form: 

\begin{proposition}[Computation of the generalized $\alpha$ coefficients]
Let $p,q$ such that $p+q\leq d$. 
Then 
\[
\alpha_{p,q} = \frac{1}{2^d} \sum_{s=0}^{d} \binom{d-p-q}{s-q}\psi\left(\frac{s}{d}\right)
\, .
\]
\end{proposition}

\begin{proof}
We follow the proof of Proposition~\ref{prop:computation-alpha-coefficients}. 
\begin{align*}
\alpha_{p,q} &= \expec{\pi z_1 \cdots z_p\cdot (1-z_{p+1})\cdots (1-z_{p+q})} \\
&= \sum_{s=0}^{d} \expecunder{\pi z_1 \cdots z_p\cdot (1-z_{p+1})\cdots (1-z_{p+q})}{s}\cdot \proba{S=s} \\
&= \frac{1}{2^d}\sum_{s=0}^{d} \binom{d}{s} \psi\left(\frac{s}{d}\right) \probaunder{z_1=\cdots=z_p=1,z_{p+1}=\cdots=z_{p+q}=0}{s} \\
&= \frac{1}{2^d}\sum_{s=0}^{d} \binom{d}{s} \psi\left(\frac{s}{d}\right) \binom{d-p-q}{s-q}\binom{d}{s} \tag{Lemma~\ref{lemma:activated-and-deactivated}} \\
\alpha_{p,q} &= \frac{1}{2^d} \sum_{s=0}^{d} \binom{d-p-q}{s-q}\psi\left(\frac{s}{d}\right)
\, .
\end{align*}
\end{proof}

Notice that the expression of $\alpha_{p,q}$ coincide with that of $\alpha_p$ when $q=0$. 
We can now give the expression of $\Gamma^f$ for an elementary shape detector in the general case.

\begin{proposition}[Computation of $\Gamma^f$, elementary shape detector]
\label{prop:gamma-computation-indicator}
Assume that $f$ is written as in Eq.~\eqref{eq:basic-shape-detector}. 
Assume that for any $j\in E_-$, $J_j\cap \Sminus=\emptyset$ (otherwise $\Gamma^f=0$). 
Let $p \defeq \card{E_-}$ and $q\defeq \card{\{j\in E_+, J_j\cap \Sminus \neq \emptyset\}}$. 
Then $\expec{\pi f(x)}=\alpha_{p,q}$ and 
\[
\expec{\pi z_j f(x)} = 
\begin{cases}
0 &\text{ if }j\in \{j\in E_+ \text{ s.t. }J_j\cap \Sminus\neq \emptyset\}\, ,\\
\alpha_{p,q} &\text{ if } j\in E_-\, , \\
\alpha_{p+1,q} &\text{ otherwise.}  
\end{cases}
\]
\end{proposition}

Taking $q=0$ (a consequence of Eq.~\eqref{eq:ass-simp-shape}) in  Proposition~\ref{prop:gamma-computation-indicator} directly yields $\expec{\pi f(x)}=\alpha_p$ and
\[
\expec{\pi z_j f(x)} = 
\begin{cases}
\alpha_p &\text{ if } j\in E_-\, , \\
\alpha_{p+1} &\text{ otherwise.}
\end{cases}
\]

\begin{proof}
We notice that, for any $u\in J_j$, 
\[
x_u = z_j \xi_u + (1-z_j)\xibar_u
\, .
\]
There are four cases to consider when deciding whether $x_u > \tau$ or not: 
\begin{itemize}
\item $\xi_u>\tau $ and $\xibar_u>\tau$, that is, $j\in E_+$ and $u\in J_j\cap \Splus$. Then $x_u > \tau$ a.s.;
\item $\xi_u \leq \tau$ and $\xibar_u>\tau$, that is, $j\in E_+$ and $u\in J_j\cap \Sminus$. Then $x_u>\tau$ if, and only if, $z_j=0$;
\item $\xi_u>\tau $ and $\xibar_u\leq \tau$, that is, $j\in E_-$ and $u\in J_j\cap \Splus$. Then $x_u > \tau$ if, and only if, $z_j=1$;
\item $\xi_u\leq \tau $ and $\xibar_u\leq \tau$, that is, $j\in E_-$ and $u\in J_j\cap \Sminus$. Then $x_u\leq \tau$ a.s., but this last case cannot happen since we assume that  for any $j\in E_-$, $J_j\cap \Sminus=\emptyset$.
\end{itemize}
This case separation allows us to rewrite $f(x)$ as
\begin{align*}
f(x) &= \prod_{u\in\shape} \indic{x_u > \tau} \tag{Eq.~\eqref{eq:basic-shape-detector}} \\
&= \prod_{j\in E_+} \prod_{u\in J_j\cap \Sminus} (1-z_j) \cdot \prod_{j\in E_-} \prod_{u\in J_j\cap \Splus} z_j 
\end{align*}
Since we assumed that for any $j\in E_-$, $J_j\cap \Sminus=\emptyset$, then for any $j\in E_-$, $J_j\cap \Splus\neq \emptyset$. 
Thus the rightmost inner products are never empty, and since $z_j\in\{0,1\}$ a.s., we deduce that there are $p$ terms in the rightmost product. 
By definition of $q$, and again since $1-z_j\in\{0,1\}$ a.s., there are $q$ terms in the leftmost product. 
By definition of $E_+$ and $E_-$, these products do not have any common terms. 
We deduce that $\expec{\pi f(x)}=\alpha_{p,q}$ by definition of the generalized $\alpha$ coefficients. 

When computing $\expec{\pi z_j f(x)}$, there are several possibilities. 
First, if $j\in \{j\in E_+ \text{ s.t. }J_j\cap \Sminus \neq \emptyset\}$, since $z_j(1-z_j)=0$ a.s., we deduce that $\expec{\pi z_j f(x)}=0$. 
Second, if $j\in E_-$, since $z_j^2=z_j$, we recover $\expec{\pi z_j f(x)}=\expec{\pi f(x)}=\alpha_{p,q}$. 
Finally, if $j$ does not belong to one of these sets, then the rightmost product gains one additional term and we obtain $\alpha_{p+1,q}$. 
\end{proof}

%%%%%%%%%%%%%%%%%%%%%%%%%%%%%%%%%%%%%%%%%%%%%%%%%%%%%%%%%%%%%%%%%%%

\subsection{Linear model}
\label{sec:gamma-linear}

In this section, we compute $\Gamma^f$ for a linear $f$. 
As in the paper, we define
\begin{equation}
\label{eq:linear-model}
f(x) = \sum_{u=1}^{D} \lambda_u x_u
\, ,
\end{equation}
with $\lambda_1,\ldots,\lambda_D\in\Reals$ arbitrary coefficients. 
By linearity, we just have to look into the case $f : x\mapsto x_u$ where $u\in\{1,\ldots,D\}$ is a fixed pixel index. 

% TODO: nice statement without restricting 
\begin{proposition}[Computation of $\Gamma^f$, linear case]
	\label{prop:gamma-computation-linear}
	Assume that $f$ is defined as in Eq.~\eqref{eq:linear-model} and $u\in J_j$.
	Then 
	\[
	\expec{\pi x_u} = \alpha_1(\xi_u - \xibar_u)+\alpha_0\xibar_u
	\, ,
	\]
	\[
\expec{\pi z_j x_u} = \alpha_1(\xi_u-\xibar_u)+\alpha_1\xibar_u
	\, ,
	\]
and, for any $k\neq j$, 
\[
\expec{\pi z_k x_u} = \alpha_2(\xi_u-\xibar_u)+\alpha_1\xibar_u
\, .
\]
\end{proposition}

\begin{proof}
As in the proof of Proposition~\ref{prop:gamma-computation-indicator}, we notice that
\[
x_u = z_j\xi_u + (1-z_j)\xibar_u
\, .
\] 
Then we write
\begin{align*}
\expec{\pi x_u} &= \expec{\pi (z_j \xi_u + (1-z_j)\xibar_u)} \\
&= \expec{\pi z_j (\xi_u - \xibar_u) + \pi \xibar_u} \\
\expec{\pi x_u} &= \alpha_1(\xi_u-\xibar_u) + \alpha_0 \xibar_u
\, ,
\end{align*}
where we used the definition of the $\alpha$ coefficients. 
Now let us compute $\expec{\pi z_j f(x)}$:
\begin{align*}
	\expec{\pi z_j x_u} &= \expec{\pi z_j(z_j \xi_u + (1-z_j)\xibar_u)} \\
	&= \expec{\pi z_j ((\xi_u-\xibar_u)z_j + \xibar_u)} \tag{$z_j\in\{0,1\}$ a.s.} \\
	\expec{\pi z_j x_u}&= \alpha_1 (\xi_u - \xibar_u) + \alpha_1\xibar_u
	\, .
\end{align*}
And finally, for any $k\neq j$,
\begin{align*}
	\expec{\pi z_k x_u} &= \expec{\pi z_k ((\xi_u-\xibar_u)z_j + \xibar_u)} \\
	&= \alpha_2(\xi_u-\xibar_u)+\alpha_1\xibar_u
	\, .
	\end{align*}
\end{proof}

%%%%%%%%%%%%%%%%%%%%%%%%%%%%%%%%%%%%%%%%%%%%%%%%%%%%%%%%%%%%%%%%%%%%

\subsection{Concentration of $\Gammahat_n$}
\label{sec:concentration-gammahat}

We now show that $\Gammahat_n$ is concentrated around $\Gamma^f$. 
Since the expression of $\Gammahat_n$ is the same than in the tabular case, and we assume that $f$ is bounded on the support of $x$, the same reasoning as in the proof of Proposition~24 in \citet{garreau_luxburg_2020_arxiv} can be applied. 

\begin{proposition}[Concentration of $\Gammahat_n$]
	\label{prop:concentration-gammahat}
	Assume that $f$ is bounded by $M>0$ on $\supp{x}$. 
	Then, for any $t>0$, it holds that 
	\[
	\proba{\smallnorm{\Gammahat_n - \Gamma^f} \geq t} \leq 4d \exp{\frac{-nt^2}{32Md^2}}
	\, .
	\]
\end{proposition}

\begin{proof}
Since $f$ is bounded by $M$ on $\supp{x}$, it holds that $\abs{f(x)}\leq M$ almost surely. 
We can then proceed as in the proof of Proposition~24 in \citet{garreau_luxburg_2020_arxiv}. 
\end{proof}

%%%%%%%%%%%%%%%%%%%%%%%%%%%%%%%%%%%%%%%%%%%%%%%%%%%%%%%%%%%%%%%%%%

\section{The study of $\beta^f$}
\label{sec:study-of-beta}

\subsection{Concentration of $\betahat_n$}

In this section we show the concentration of $\betahat_n$ (Theorem~1 in the paper). 
The proof scheme follows closely that of \citet{garreau_luxburg_2020_arxiv}. 

\begin{theorem}[Concentration of $\betahat_n$]
	\label{th:concentration-betahat}
Assume that $f$ is bounded by a constant $M$ on the unit cube $[0,1]^D$. 
Let $\epsilon > 0$ and $\eta\in (0,1)$. 
Let $d$ be the number of superpixels used by LIME. 
Then, there exists $\beta^f\in\Reals^{d+1}$ such that, for every
\[
n \geq \ceil*{\max\left(2^{15}d^4\exps{\frac{2}{\nu^2}},\frac{2^{21}d^7\max(M,M^2)\exps{\frac{4}{\nu^2}}}{\epsilon^2}\right) \log\frac{8d}{\eta}}
\, ,
\]
we have $\smallproba{\smallnorm{\betahat_n-\beta^f} \geq \epsilon}\leq \eta$. 
\end{theorem}

\begin{proof}
As in \citet{garreau_luxburg_2020_arxiv}, the key idea of the proof is to notice that 
\begin{equation}
\label{eq:binding-lemma}
\smallnorm{\betahat_n-\beta^f} \leq 2\opnorm{\Sigma^{-1}} \smallnorm{\Gammahat-\Gamma^f} + 2\opnorm{\Sigma^{-1}}^2 \norm{\Gamma^f}\smallopnorm{\Sigmahat-\Sigma} 
\, ,
\end{equation}
provided that (i) $\smallopnorm{\Sigma^{-1}(\Sigmahat-\Sigma)}\leq 0.32$ (this is Lemma~27 in \citet{garreau_luxburg_2020_arxiv}. 
We are going to build an event of probability at least $1-\eta$ such that $\Sigmahat_n$ is close to $\Sigma$ and $\Gammahat_n$ is close from $\Gamma^f$. 
The deterministic bound obtained on $\opnorm{\Sigma^{-1}}$ together with the boundedness of $f$ will allow us to show that (ii) $\opnorm{\Sigma^{-1}} \smallnorm{\Gammahat-\Gamma^f}\leq \epsilon/4$ and (iii) $\opnorm{\Sigma^{-1}}^2 \norm{\Gamma^f}\smallopnorm{\Sigmahat-\Sigma} \leq \epsilon / 4$. 

We first show (i). 
Let us set $n_1\defeq \ceil*{2^{15} d^4 \exps{\frac{2}{\nu^2}} \log \frac{8d}{\eta}}$ and $t_1\defeq \frac{1}{25d\exps{\frac{1}{\nu^2}}}$. 
According to Proposition~\ref{prop:sigmahat-concentration}, for any $n\geq n_1$, 
\[
\proba{\opnorm{\Sigmahat_n-\Sigma}\geq t_1} \leq 4d\cdot \exp{\frac{-n_1t_1^2}{32d^2}} \leq \frac{\eta}{2}
\, .
\]
Moreover, we know that $\opnorm{\Sigma^{-1}}\leq 8d\exps{\frac{1}{\nu^2}}$ (Corollary~\ref{cor:control-opnorm}). 
Since the operator norm is sub-multiplicative, with probability greater than $1-\eta/2$, we have
\[
\opnorm{\Sigma^{-1}(\Sigmahat_n-\Sigma)} \leq \opnorm{\Sigma^{-1}}\cdot \opnorm{\Sigmahat_n-\Sigma} \leq 8d\exps{\frac{1}{\nu^2}} \cdot t_1 = 0.32
\, .
\]

Now let us show (ii). 
Let us define $n_2\defeq \ceil*{\frac{2^{15}Md^4\exps{\frac{2}{\nu^2}}}{\epsilon^2}\log \frac{8d}{\eta}}$ and $t_2\defeq \frac{\epsilon}{32d\exps{\frac{1}{\nu^2}}}$. 
According to Proposition~\ref{prop:concentration-gammahat}, for any $n\geq n_2$, we have
\[
\proba{\norm{\Gammahat_n-\Gamma} \geq t_2} \leq 4d\cdot \exp{\frac{-n_2t_2^2}{32Md^2}} \leq \frac{\eta}{2}
\, .
\]
Recall that $\opnorm{\Sigma^{-1}}\leq 8d\exps{\frac{1}{\nu^2}}$ (Corollary~\ref{cor:control-opnorm}): with probability higher than $1-\eta/2$, 
\[
\opnorm{\Sigma^{-1}}\cdot \norm{\Gammahat_n-\Gamma^f} \leq 8d\exps{\frac{1}{\nu^2}} \cdot t_2 = \frac{\epsilon}{4}
\, .
\]

Finally let us show (iii). 
Let us define $n_3\defeq \ceil*{\frac{2^{21}d^7M^2\exps{\frac{4}{\nu^2}}}{\epsilon^2}\log \frac{8d}{\eta} }$ and $t_3\defeq \frac{\epsilon}{2^8Md^{5/2}\exps{\frac{2}{\nu^2}}}$. 
According to Proposition~\ref{prop:sigmahat-concentration}, for any $n\geq n_3$, we have
\[
\proba{\opnorm{\Sigmahat_n-\Sigma} \geq t_3} \leq 4d\cdot \exp{\frac{-n_3t_3^2}{32d^2}} \leq \frac{\eta}{2}
\, .
\]
Since $f$ is bounded by $M$, it is straightforward to show that $\norm{\Gammahat^f}\leq M\cdot d^{1/2}$. 
Moreover, recall that $\opnorm{\Sigma^{-1}}^2\leq 64d^2\exps{\frac{2}{\nu^2}}$. 
We deduce that, with probability at least $\eta/2$, 
\[
\opnorm{\Sigma^{-1}}^2 \cdot \norm{\Gamma^f}\cdot \opnorm{\Sigmahat_n-\Sigma} \leq 64d^2\exps{\frac{2}{\nu^2}} \cdot Md^{1/2} \cdot t_3 = \frac{\epsilon}{4}
\, .
\]

Finally, we notice that both $n_2$ and $n_3$ are smaller than 
\[
n_4 \defeq \ceil*{\frac{2^{21}d^7\max(M,M^2)\exps{\frac{4}{\nu^2}}}{\epsilon^2} \log \frac{8d}{\eta}}
\, .
\]
Thus (ii) and (ii) simultaneously happen on an event of probability greater than $\eta/2$ when $n$ is larger than $n_4$. 
We conclude by a union bound argument. 
\end{proof}

\begin{remark}
In view of Remark~\ref{rk:tight-opnorm}, it seems difficult to improve much the rate of convergence given by Theorem~\ref{th:concentration-betahat} with the current proof technology. 
Indeed, a careful inspection of the proof reveals that, starting from Eq.~\eqref{eq:binding-lemma}, the control of $\opnorm{\Sigma^{-1}}$ is key. 
Since the dependency in $d$ seems tight, there is not much hope for improvement. 
\end{remark}

%%%%%%%%%%%%%%%%%%%%%%%%%%%%%%%%%%%%%%%%%%%%%%%%%%%%%%%%%%%%%%%%%%%%

\subsection{General expression of $\beta^f$}

We are now able to recover Proposition~2 of the paper: the expression of $\beta^f$ is obtained simply by multiplying  Eq.~\eqref{eq:sigma-inverse-computation} and~\eqref{eq:def-gamma}. 
We also give the value of the intercept ($\beta_0$ with our notation), which is omitted in the paper for simplicity's sake.

\begin{corollary}[Computation of $\beta^f$]
	\label{cor:computation-beta-general}
%	Assume that $f$ is bounded by a constant $M>0$ on $[0,1]^D$. 
Under the assumptions of Theorem~\ref{th:concentration-betahat}. 
	\begin{equation}
	\label{eq:beta-computation-intercept}
	\beta^f_0 = \dencst_d^{-1}\biggl\{\sigma_0\expec{\pi f(x)} + \sigma_1\sum_{j=1}^d \expec{\pi z_{j} f(x)}\biggr\}
	\, ,
	\end{equation}
	and, for any $1\leq j\leq d$, 
	\begin{equation}
	\label{eq:beta-computation-general}
	\beta^f_j = 
	\dencst_d^{-1}\biggl\{\sigma_1 \expec{\pi f(x)} + \sigma_2 \expec{\pi z_j f(x)} + \sigma_3 \sum_{\substack{k=1 \\ k\neq j}}^d \expec{\pi z_k f(x)}\biggr\}
	\, .
	\end{equation}
\end{corollary}

%%%%%%%%%%%%%%%%%%%%%%%%%%%%%%%%%%%%%%%%%%%%%%%%%%%%%%%%%%%%%%%%%%%%%

\subsection{Shape detectors}

We now specialize Corollary~\ref{cor:computation-beta-general} to the case of elementary shape detectors. 

\begin{proposition}[Expression of $\beta^f$, shape detector]
\label{prop:beta-computation-shape-detector}
Let $f$ be written as in Eq.~\eqref{eq:basic-shape-detector}. 
Assume that for any $j\in E_-$, $J_j\cap \Sminus=\emptyset$ (otherwise $\beta^f=0$). 
Let $p$ and $q$ as before. 
Then 
\[
\beta_0^f = \dencst_d^{-1} \left\{\sigma_0\alpha_{p,q}+ p\sigma_1\alpha_{p,q} + (d-p-q)\alpha_{p+1,q} \right\}
\, ,
\]
for any $j\in E_-$, 
\[
\beta_j^f = \dencst_d^{-1}\left\{\sigma_1\alpha_{p,q} + \sigma_2\alpha_{p,q} + (p-1)\sigma_2\alpha_{p,q} + (d-p-q)\sigma_3\alpha_{p+1,q} \right\}
\, ,
\]
for any $j\in E_+$ such that $J_j\cap \Sminus\neq \emptyset$,
\[
\beta_j^f = \dencst_d^{-1}\left\{ \sigma_1\alpha_{p,q} + p\sigma_3\alpha_{p,q} + (d-p-q)\alpha_{p+1,q} \right\}
\, ,
\]
and 
\[
\beta_j^f = \dencst_d^{-1}\left\{\sigma_1\alpha_{p,q} + \sigma_2\alpha_{p+1,q} + p\sigma_3\alpha_{p,q} + (d-p-q-1)\sigma_3\alpha_{p+1,q} \right\}
\]
otherwise.
\end{proposition}

\begin{proof}
Straightforward from Corollary~\ref{cor:computation-beta-general} and Proposition~\ref{prop:gamma-computation-indicator}. 
\end{proof}

Note that taking $q=0$ in Proposition~\ref{prop:beta-computation-shape-detector} yields Proposition~3 of the paper. 

%%%%%%%%%%%%%%%%%%%%%%%%%%%%%%%%%%%%%%%%%%%%%%%%%%%%%%%%%%%%%%%%%

\subsection{Linear models}

We deduce from Proposition~\ref{prop:gamma-computation-linear} the expression of $\beta^f$ for linear models. 
Let us define $M_j$ the binary mask associated to superpixel $J_j$ and let $\circ$ be the termwise product. 

\begin{proposition}[Computation of $\beta^f$, linear case]
\label{prop:computation-beta-linear}
Assume that $f$ is defined as in Eq.~\eqref{eq:linear-model}. 
Then 
	\[
	\beta_0^f = \sum_{u=1}^D \lambda_u \xibar_u = f(\xibar)
	\, ,
	\]
	and, for any $1\leq j\leq d$,
	\[
	\beta_j^f = \sum_{u\in J_j} \lambda_u (\xi_u - \xibar_u)= f(M_j\circ (\xi-\xibar))
	\, .
	\]
\end{proposition}

It is interesting to compute prediction of the surrogate model at $\xi$:
\[
\beta^f_0 + \beta_1^f + \cdots + \beta_d^f = f(\xibar) + f(M_1\circ (\xi-\xibar)) + \cdots + f(M_d\circ (\xi-\xibar)) = f(\xi)
\, .
\]
Thus in the case of linear models, the limit explanation is faithful.

\begin{proof}
By linearity, we can start by computing $\beta^f$ for the function $x\mapsto x_u$. 
Assume that $j\in\{1,\ldots,d\}$ is such that $u\in J_j$. 
	According to Corollary~\ref{cor:computation-beta-general} and Proposition~\ref{prop:gamma-computation-linear}, 
	\begin{align*}
	\beta_0^f &= \frac{1}{\dencst_d}\biggl\{\sigma_0\expec{\pi f(x)} + \sigma_1\sum_{j=1}^d \expec{\pi z_{j} f(x)}\biggr\} \\
	&= \frac{1}{\dencst_d}\biggl\{\sigma_0(\alpha_1(\xi_u-\xibar_u)+\alpha_0\xibar_u) + \sigma_1(\alpha_1(\xi_u-\xibar_u)+\alpha_1\xibar_u) + (d-1)\sigma_1(\alpha_2(\xi_u-\xibar_u)+\alpha_1\xibar_u)\biggr\} \\
	&= \frac{1}{\dencst_d}\biggl\{(\sigma_0\alpha_1 + \sigma_1\alpha_1+(d-1)\sigma_1\alpha_2)(\xi_u-\xibar_u) +(\sigma_0\alpha_0+d\sigma_1\alpha_1)\xibar_u\biggr\} \\
	\beta_0^f &= \xibar_u 
	\, ,
	\end{align*}
where we used Eqs.~\eqref{eq:aux-rel-0} and~\eqref{eq:aux-rel-4} in the last display. 
	\begin{align*}
	\beta_j^f &= \frac{1}{\dencst_d}\biggl\{\sigma_1 \expec{\pi f(x)} + \sigma_2 \expec{\pi z_j f(x)} + \sigma_3 \sum_{\substack{k=1 \\ k\neq j}}^d \expec{\pi z_k f(x)} \biggr\} \\
	&= \frac{1}{\dencst_d}\biggl\{ \sigma_1(\alpha_1(\xi_u-\xibar_u)+\alpha_0\xibar_u) + \sigma_2(\alpha_1(\xi_u-\xibar_u)+\alpha_1\xibar_u) + (d-1)\sigma_3(\alpha_2(\xi_u-\xibar_u)+\alpha_1\xibar_u) \biggr\}\\
	&= \frac{1}{\dencst_d}\biggl\{(\sigma_1\alpha_1+\sigma_2\alpha_1+(d-1)\sigma_3\alpha_2)(\xi_u-\xibar_u)+(\sigma_1\alpha_0 +\sigma_2\alpha_1 +(d-1)\sigma_3\alpha_1)\xibar_u \biggr\} \\
	\beta_j^f &= \xi_u - \xibar_u
	\, ,
	\end{align*}
where we used Eqs.~\eqref{eq:aux-rel-1} and~\eqref{eq:aux-rel-3} in the last display. 
	Finally, let $k\neq j$:
	\begin{align*}
	\beta_k^f &= \frac{1}{\dencst_d}\biggl\{\sigma_1 \expec{\pi f(x)} + \sigma_2 \expec{\pi z_k f(x)} + \sigma_3 \sum_{\substack{k'=1 \\ k'\neq j,k}}^d \expec{\pi z_{k'} f(x)} \biggr\} \\
	&= \frac{1}{\dencst_d}\biggl\{\sigma_1(\alpha_1(\xi_u-\xibar_u)+\alpha_0\xibar_u) +\sigma_2(\alpha_2(\xi_u-\xibar_u)+\alpha_1\xibar_u)+\sigma_3(\alpha_1(\xi_u-\xibar_u)+\alpha_1\xibar_u) \\
	&\qquad \qquad \qquad \qquad \qquad \qquad + (d-2)\sigma_3(\alpha_2(\xi_u-\xibar_u)+\alpha_1\xibar_u)\biggr\} \\
	&= \frac{1}{\dencst_d}\biggl\{(\sigma_1\alpha_1+\sigma_2\alpha_2+\sigma_3\alpha_1+(d-2)\sigma_3\alpha_2)(\xi_u-\xibar_u) + (\sigma_1\alpha_0+\sigma_2\alpha_1+(d-1)\sigma_3\alpha_1)\xibar_u \biggr\} \\
	\beta_k^f &= 0 
	\, ,
	\end{align*}
where we used Eqs.~\eqref{eq:aux-rel-2} and~\eqref{eq:aux-rel-3} in the last display. 
	We deduce the result by linearity. 
\end{proof}

%%%%%%%%%%%%%%%%%%%%%%%%%%%%%%%%%%%%%%%%%%%%%%%%%%%%%%%%%%%%%%%%%%%%

\section{Technical results}
\label{sec:technical}

\subsection{Probability computations}

In this section we collect all elementary probability computations necessary for the computation of the $\alpha$ coefficients and the generalized $\alpha$ coefficients. 

\begin{lemma}[Activated only]
	\label{lemma:basic-computations}
	Let $p\geq 0$ be an integer. Then
	\[
	\probaunder{z_1=1,\ldots,z_p=1}{s} = \frac{(d-p)!}{d!}\cdot \frac{(d-s)!}{(d-s-p)!}
	\, .
	\]
\end{lemma}

\begin{proof}
Conditionally to $S=s$, the choice of $S$ is uniform among all subsets of $\{1,\ldots,d\}$. 
Therefore we recover the proof of Lemma~4 in~\citet{mardaoui2020analysis}. 
\end{proof}

The following lemma is a slight generalization, which coincides when $q=0$.

\begin{lemma}[Activated and deactivated]
\label{lemma:activated-and-deactivated}
Let $p,q$ be integers. 
Then
\[
\probaunder{z_1=\cdots=z_p=1,z_{p+1}=\cdots=z_{p+q}=0}{s} = \binom{d-p-q}{s-q} \binom{d}{s}^{-1}
\, .
\]
\end{lemma}

\begin{proof}
Conditionally to $S=s$, the deletions are uniformly distributed. 
Therefore, the total number of cases is $\binom{d}{s}$. 
Now, the favorable cases correspond to superpixels $p+1,\ldots,p+q$ deleted: these are $q$ fixed deletions. 
We also need to have superpixels $1,\ldots,p$ activated, these are $p$ indices that are not available to deletions. 
In total, we need to place $s-q$ deletions among $d-p-q$ possibilities. 
We deduce the result. 
\end{proof}

%%%%%%%%%%%%%%%%%%%%%%%%%%%%%%%%%%%%%%%%%%%%%%%%%%%%%%%%%%%%%%%%%%

\subsection{Algebraic identities}

In this section we collect some identities used throughout the proofs.

\begin{proposition}[Four letter identity]
	\label{prop:four-letter}
	Let $A$, $B$, $C$, and $D$ be four finite sequences of real numbers. 
	Then it holds that
	\[
	\sum_j A_j C_j \cdot \sum_j B_jD_j - \sum_j A_jB_j \cdot \sum C_j D_j = \sum_{j<k} (A_jD_k-A_kD_j)(C_jB_k-C_kB_j)
	\, .
	\]
\end{proposition}

\begin{proof}
	See the proof of Exercise~3.7 in \citet{steele2004cauchy}. 
\end{proof}

\begin{proposition}[A combinatorial identity]
	\label{prop:combinatorial-identity}
	Let $d\geq 1$ be an integer. 
	Then
	\[
	V_d \defeq \sum_{j<k} \binom{d}{j}\binom{d}{k}(j-k)^2 = d\cdot 4^{d-1}
	\, .
	\]
\end{proposition}

\begin{proof}
	We first notice that
	\begin{align*}
	V_d &= \frac{1}{2}\sum_{j,k} \binom{d}{j}\binom{d}{k}(j-k)^2 \tag{by symmetry} \\
	&= \sum_{j,k} \binom{d}{j}\binom{d}{k} k^2 - \sum_{j,k} \binom{d}{j}\binom{d}{k} jk \tag{developing the square} \\
	&= \sum_j \binom{d}{j} \sum_k \binom{d}{k}k^2 - \left(\sum_j \binom{d}{j}j\right)^2
	\, .
	\end{align*}
	It is straightforward to show that 
	\[
	\sum_j \binom{d}{j} = 2^d\, ,
	\sum_j \binom{d}{j} j = d\cdot 2^{d-1}\, ,
	\text{ and }
	\sum_j \binom{d}{j} j^2 = d(d+1)\cdot 2^{d-2}
	\, .
	\]
	We deduce that 
	\begin{align*}
	\dencst_d &= 2^d \cdot d(d+1)\cdot 2^{d-2} - d^2 \cdot 2^{2d-2} = d \cdot 4^{d-1}
	\, .
	\end{align*}
\end{proof}

% % % % % % % % % % % % % % % % % % % % % % % % % % % % % % % % % % % % % % % %

\section{Additional results}
\label{sec:experiments}

In this section, we present additional qualitative results on the three pre-trained models used in the paper: MobileNetV2~\citep{sandler2018mobilenetv2}, DenseNet121~\citep{huang2017densely}, and InceptionV3~\citep{szegedy2016rethinking}. 
%For each model, we selected $8$ images from t

\begin{figure}
	\begin{center}
\hspace{-0.5cm}\includegraphics[scale=0.45]{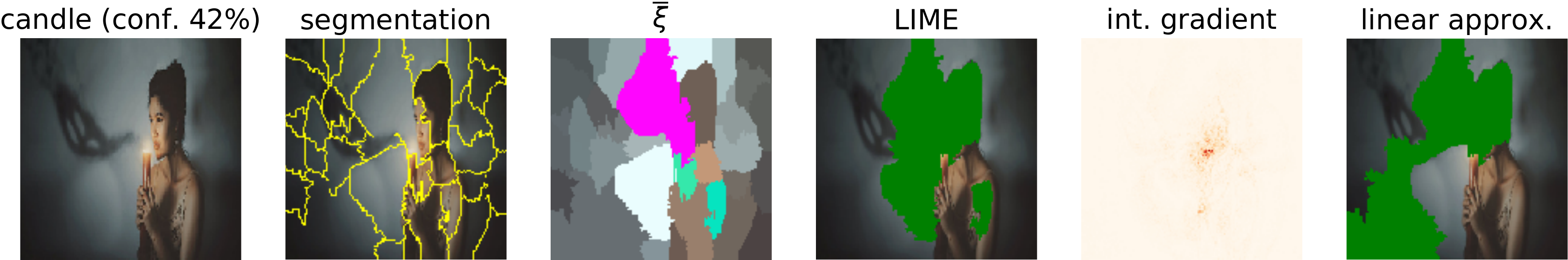} 
		
		\hspace{-0.53cm}\includegraphics[scale=0.45]{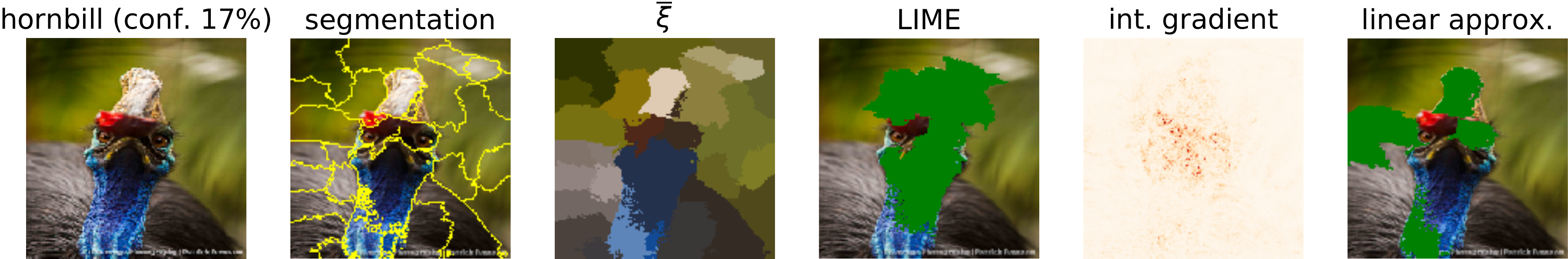}
		
		\hspace{-0.21cm}\includegraphics[scale=0.45]{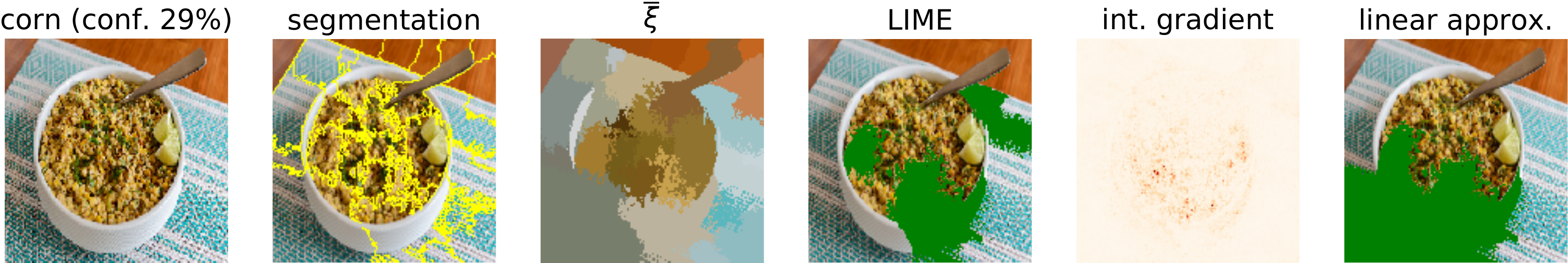}
		
		\hspace{-0.3cm}\includegraphics[scale=0.45]{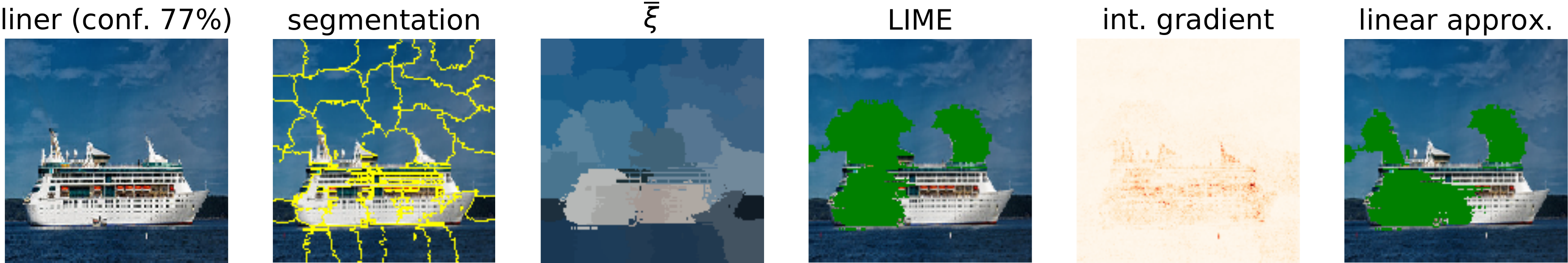}
		
		\hspace{-0.6cm}\includegraphics[scale=0.45]{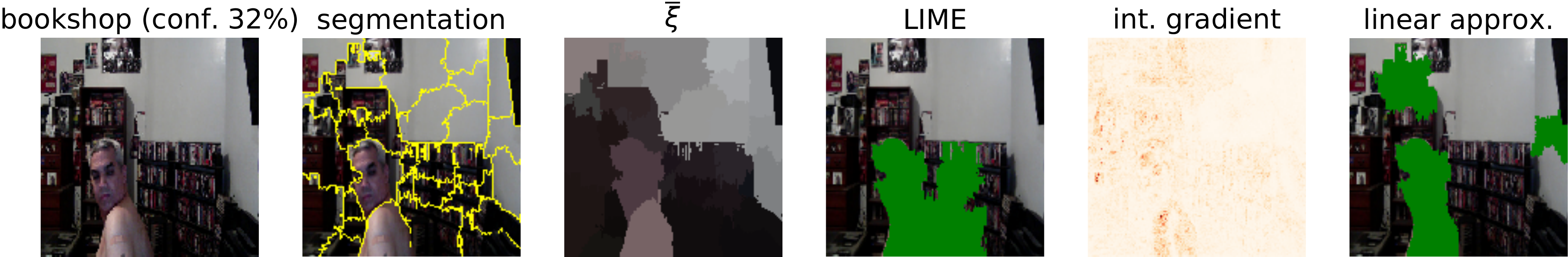}
		
		\hspace{-0.7cm}\includegraphics[scale=0.45]{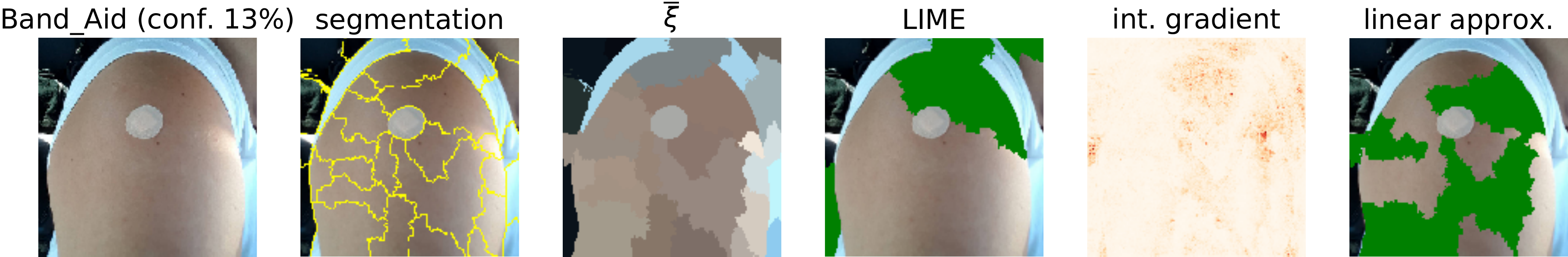}
		
		\hspace{-0.3cm}\includegraphics[scale=0.45]{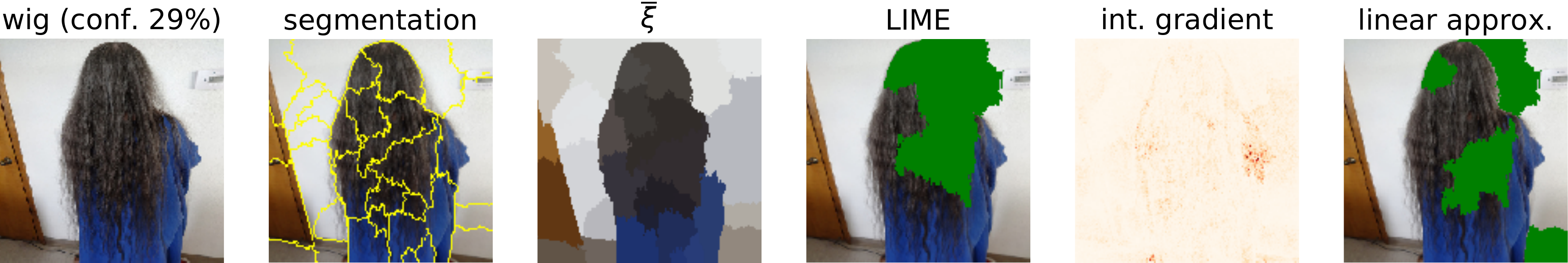}
		
		\hspace{-1.1cm}\includegraphics[scale=0.45]{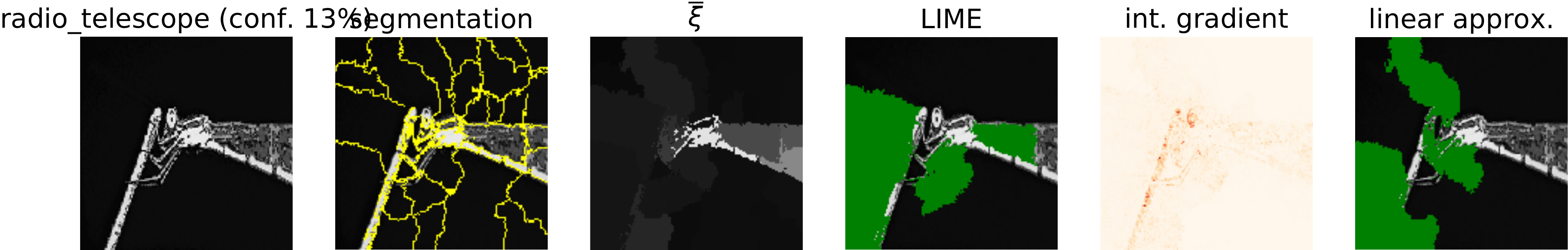}
	\end{center}
	\caption{Empirical explanations, integrated gradient, and approximated explanations for images from the ILSVRC2017 dataset. The model explained is the likelihood function associated to the top class given by MobileNetV2.}
\end{figure}

\begin{figure}
	\begin{center}
		\hspace{-0.5cm}\includegraphics[scale=0.45]{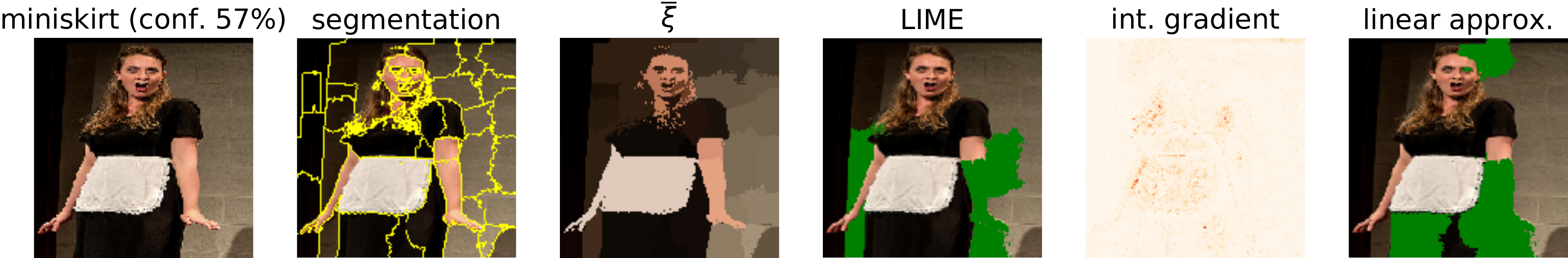} 
		
		\hspace{-0.55cm}\includegraphics[scale=0.45]{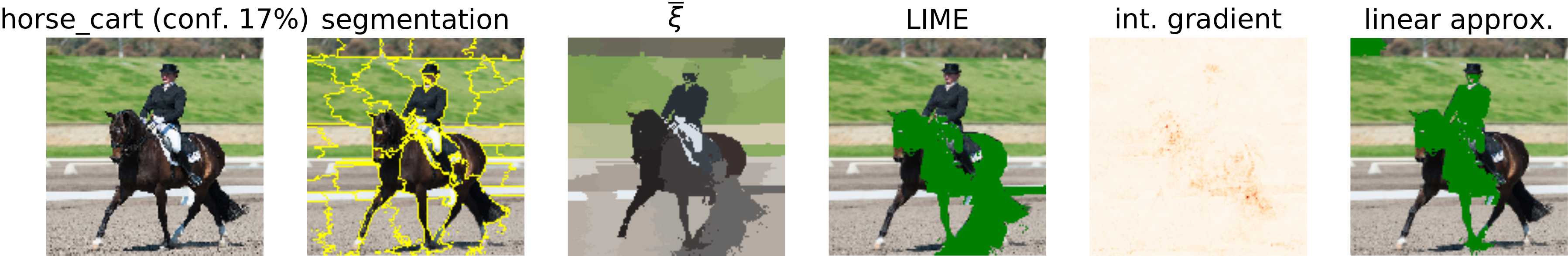}
		
		\hspace{-0.75cm}\includegraphics[scale=0.45]{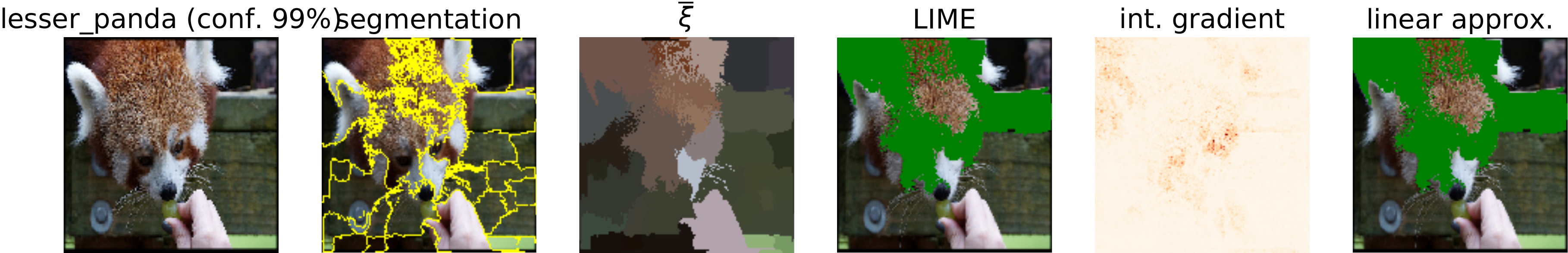}
		
		\hspace{-0.25cm}\includegraphics[scale=0.45]{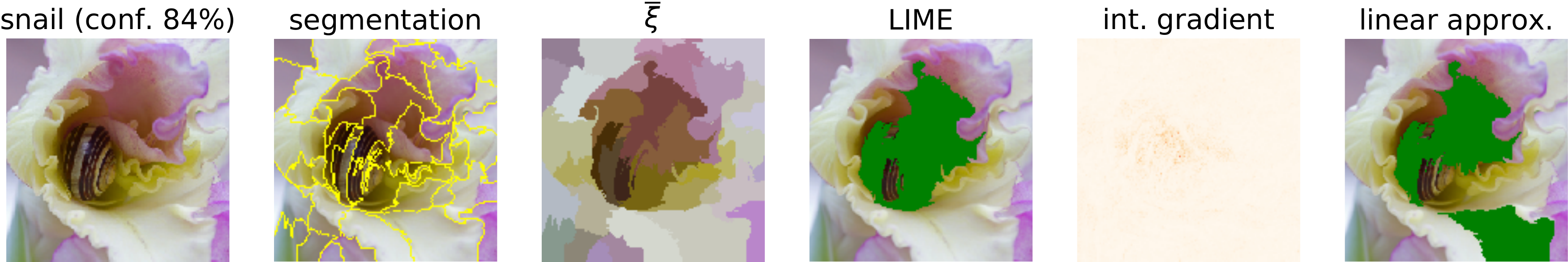}
		
		\hspace{-0.5cm}\includegraphics[scale=0.45]{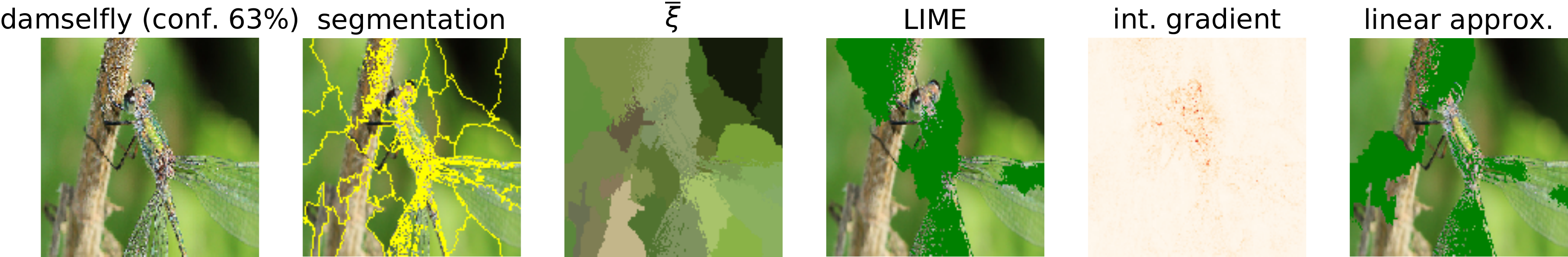}
		
		\hspace{-0.3cm}\includegraphics[scale=0.45]{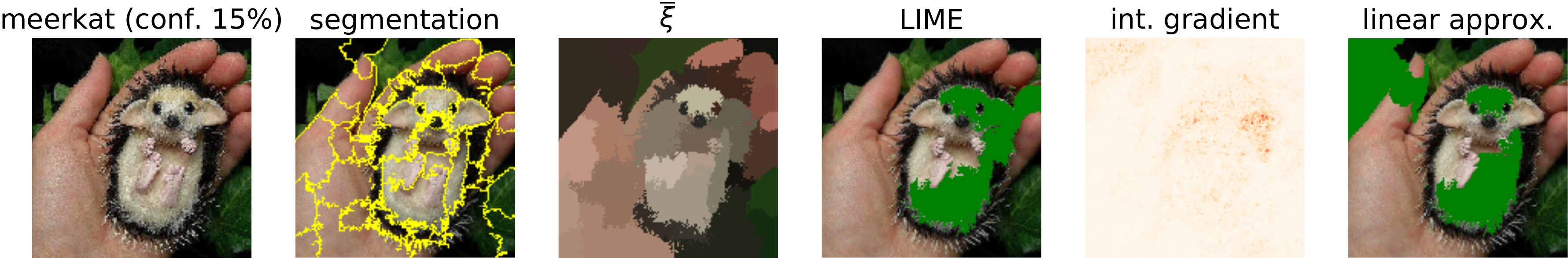}
		
		\hspace{-0.0cm}\includegraphics[scale=0.45]{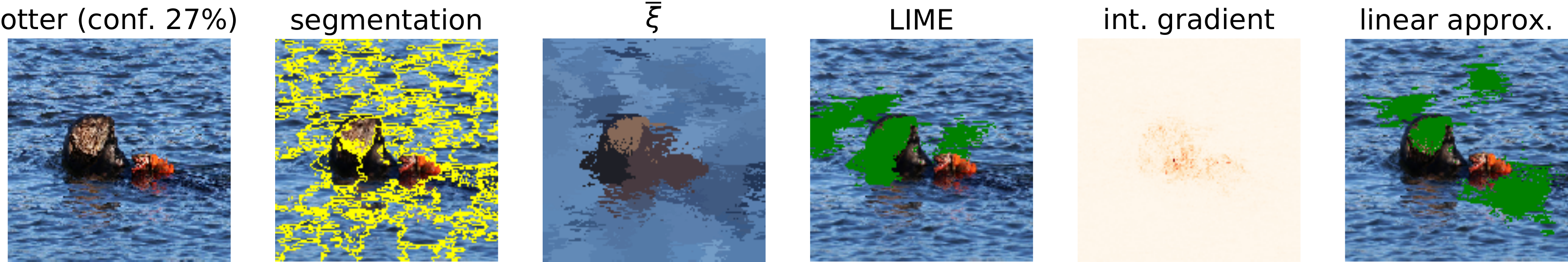}
		
		\hspace{-0.58cm}\includegraphics[scale=0.45]{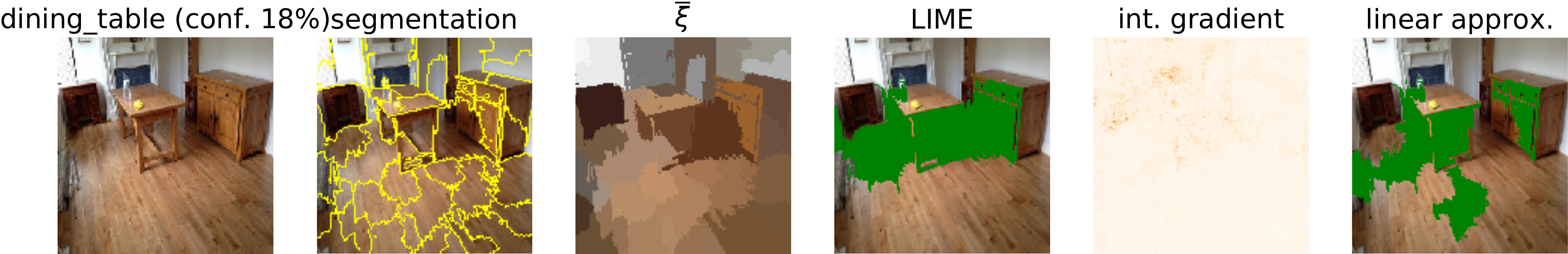}
	\end{center}
	\caption{Empirical explanations, integrated gradient, and approximated explanations for images from the ILSVRC2017 dataset. The model explained is the likelihood function associated to the top class given by DenseNet121.}
\end{figure}

\begin{figure}
\begin{center}
\includegraphics[scale=0.45]{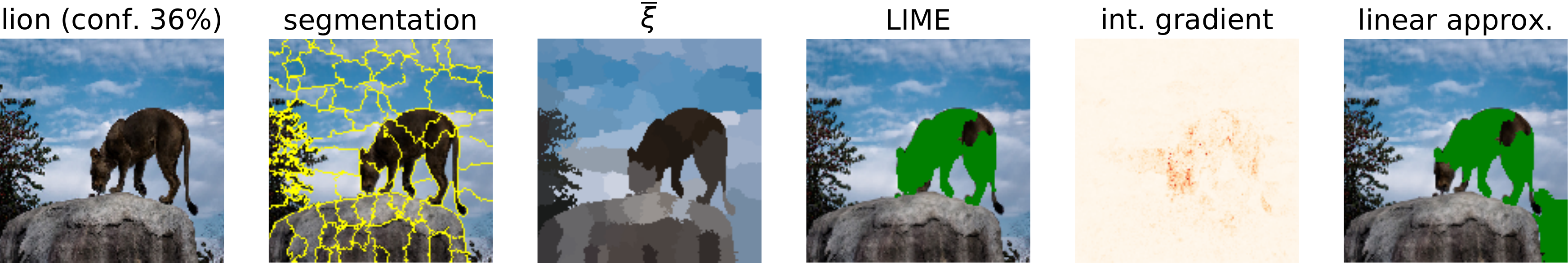} 

\hspace{-0.53cm}\includegraphics[scale=0.45]{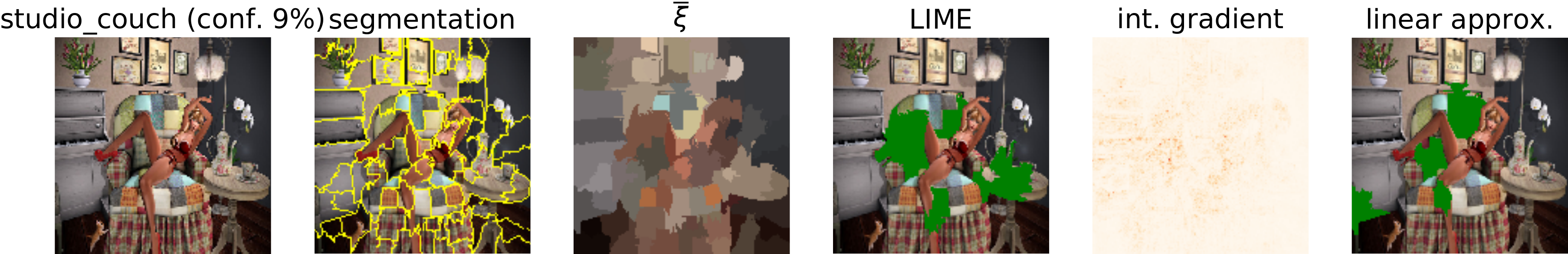}

\hspace{-0.2cm}\includegraphics[scale=0.45]{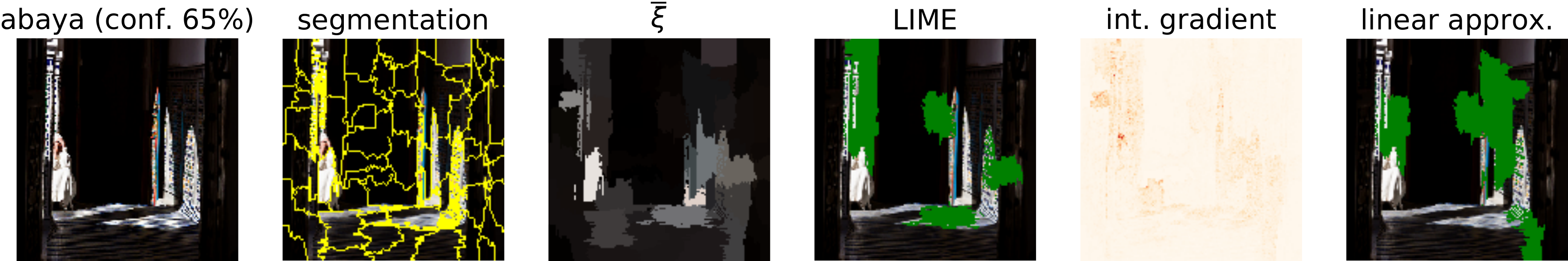}

\hspace{-0.3cm}\includegraphics[scale=0.45]{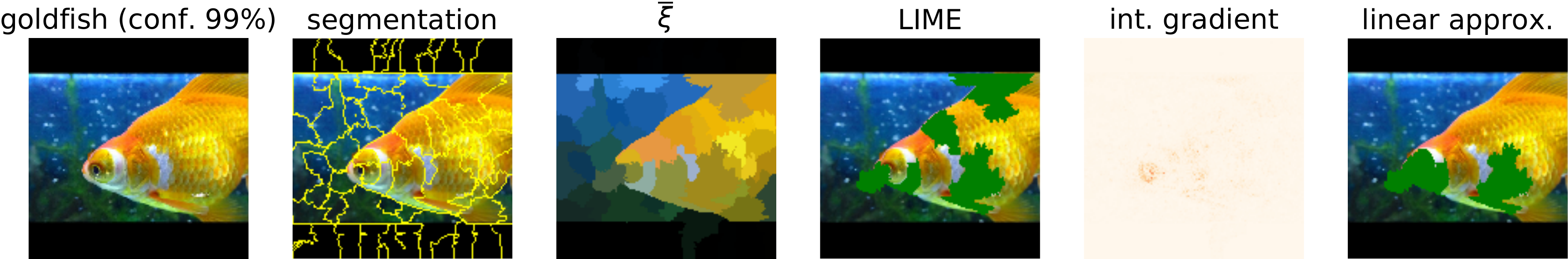}

\hspace{-0.6cm}\includegraphics[scale=0.45]{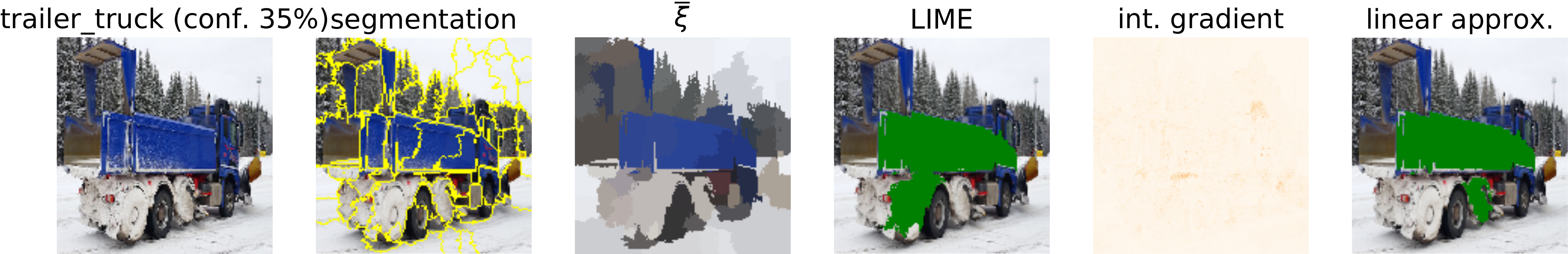}

\hspace{-0.7cm}\includegraphics[scale=0.45]{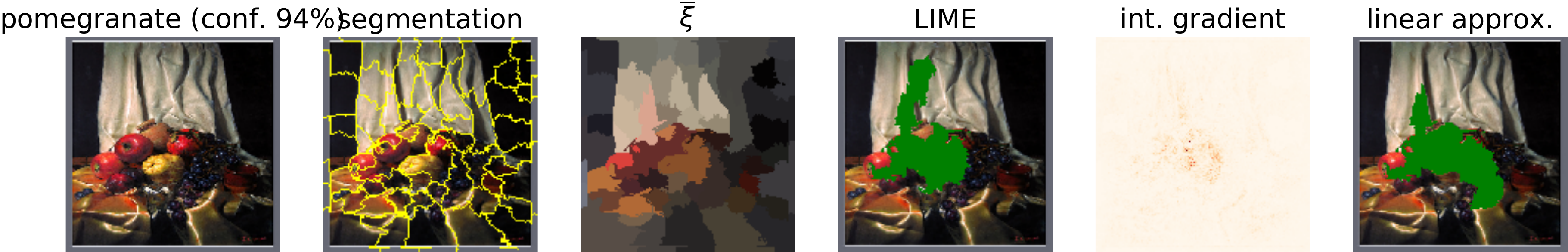}

\hspace{-0.1cm}\includegraphics[scale=0.45]{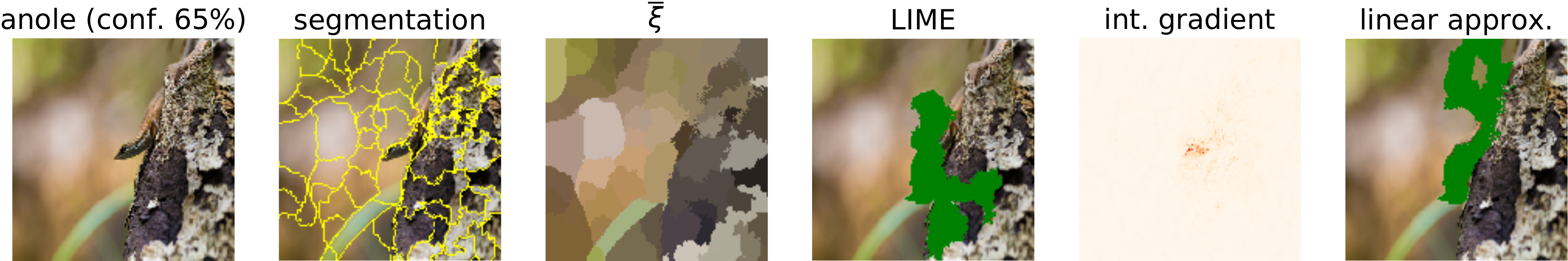}

\hspace{-0.55cm}\includegraphics[scale=0.45]{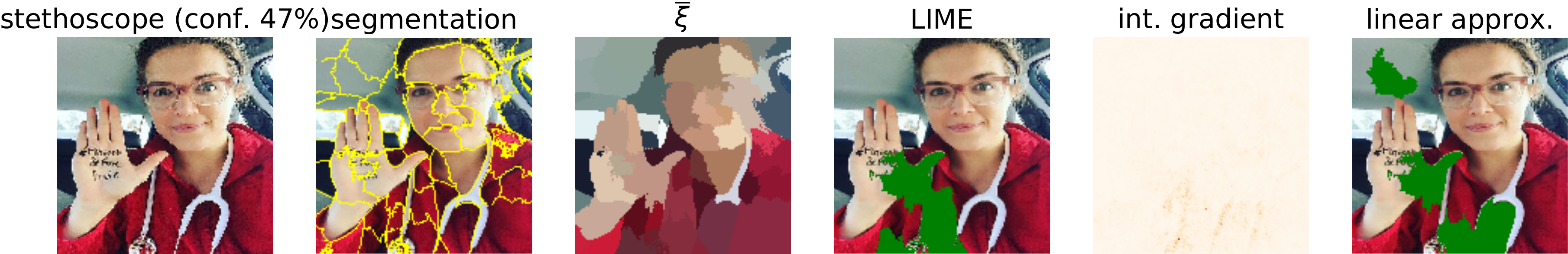}
\end{center}
\caption{Empirical explanations, integrated gradient, and approximated explanations for images from the ILSVRC2017 dataset. The model explained is the likelihood function associated to the top class given by InceptionV3.}
\end{figure}

%\newpage
%\bibliographystyle{abbrvnat}
%\bibliography{../biblio}

\end{document}